  \documentclass{colt2013}

\newcommand{\omitfromsubmission}[1]{}

\newcommand{\forspace}[1]{}

\usepackage{color}
\usepackage{float}
\usepackage{times}
\usepackage{amsmath,amssymb,amsfonts}

\newcommand{\mycomment}[1]{}

\newcommand{\repeatclaim}[2]{
    \addvspace{5pt plus 3pt minus 1pt}
    \noindent\textbf{#1. }{\it #2}

    \addvspace{5pt plus 3pt minus 5pt}
}

\floatstyle{ruled}
\newfloat{algorithm}{h}{loa}
\floatname{algorithm}{Algorithm}

\newcommand{\qed}{\hfill\rule{1ex}{1.5ex}}
\newcommand{\Var}{{\rm Var}}
\def\norm#1{\left\| #1 \right\|}
 \newenvironment{sketch}{\par{\sc Proof Sketch:}} {\qed \par}

\newcommand{\D}{\mathcal D} % joint distribution over \X \times \Y
\newcommand{\DIS}{{\rm DIS}}
\renewcommand{\H}{\mathcal H} % hypothesis space
\newcommand{\C}{\mathbb C} % concept space
\renewcommand{\P}{\mathbb P} % also probability
\renewcommand{\Pr}{\mathbb P} % also probability
\newcommand{\R}{\mathit R} % overall risk functional
\newcommand{\reals}{\mathbb{R}} % real numbers

\newcommand{\E}{\mathbb E}

\newcommand{\cY}{{ Y}}

\newcommand{\err}{\mathrm{err}}
\newcommand{\dis}{\mathrm{dis}}
\newcommand{\capacity}{\mathrm{cap}}
\newcommand{\distf}{\mathrm{d}}

\newcommand{\vol}{\mathrm{vol}}

\newcommand{\Uball}{\mathrm{UBALL}}

\newcommand{\poly}{\mathrm{poly}}

\newcommand{\anglesep}{\theta}

\newcommand{\roundsEPSwc}{{\lceil\log_2{\frac{1}{c\epsilon}}\rceil}}

\newcommand{\sgn}{{\mathrm{ sign}}}
\newcommand{\sign}{{\mathrm{ sign}}}

\newcommand{\betats}{a}

\newcommand{\tc}{\tilde{c}}
\newcommand{\distrib}{D}

\newcommand{\rounds}{s}
\newcommand{\integers}{Z^+}

\title[Active and passive
   learning of linear separators
   under log-concave distributions]{Active and passive
   learning of linear separators \\
   under log-concave distributions \\}

 \coltauthor{\Name{Maria Florina Balcan} \Email{ninamf@cc.gatech.edu}\\
 \addr Georgia Institute of Technology
 \AND
 \Name{Philip M. Long} \Email{plong@microsoft.com}\\
 \addr Microsoft
 }

\begin{document}

\maketitle

\begin{abstract}
We provide new results concerning label efficient, polynomial time, passive and active learning of linear separators.
 We prove that
active learning provides an exponential improvement over  PAC (passive) learning of  homogeneous
linear separators under nearly log-concave distributions.
Building on this, we  provide a computationally efficient PAC algorithm with optimal (up to a constant factor) sample complexity for such problems. This resolves an open question of~\citep{Lon95,Lon03,phil-doubling} concerning the sample complexity of efficient PAC algorithms under the uniform distribution in the unit ball. Moreover,
it provides the first bound for a polynomial-time  PAC algorithm that is tight for an interesting infinite class of hypothesis functions under a
general and natural class of data-distributions, providing significant progress towards  a longstanding open question of~\citep{EHKV89,BEHW89}.

We also provide new bounds for active and passive learning in the case that the
data might not be linearly separable, both in the agnostic case and and under the Tsybakov low-noise condition.
To derive our results, we provide new
structural results for (nearly) log-concave distributions, which might be of independent interest  as well.
\end{abstract}
%
%\begin{keywords}
%Active learning,  PAC learning, nearly log-concave distributions, noise-tolerant learning.\end{keywords}

\begin{keywords}
Active learning,  PAC learning, ERM, nearly log-concave distributions, Tsybakov low-noise condition, agnostic learning.\end{keywords}
%Active learning
%PAC learning
%linear separators
%log-concave distributions
%nearly log-concave distributions

\vspace{-3mm}
\section{Introduction}
\vspace{-1mm}
Learning linear separators is  one of the central challenges in machine learning.
They are widely used and have been long studied both in the statistical and computational learning theory.
A seminal result of \citep{BEHW89}, using tools due to
\citep{VC}, showed that $d$-dimensional linear separators
can be learned to accuracy $1 - \epsilon$ with probability $1 - \delta$
in the classic PAC model in polynomial time with $O((d/\epsilon) \log
(1/\epsilon) + (1/\epsilon)\log(1/\delta))$ examples.
The best known lower bound for
 linear separators is $\Omega(d/\epsilon + (1/\epsilon)\log(1/\delta))$, and this holds even in the case in which the distribution is uniform~\citep{Lon95}.
Whether the upper bound can be improved to match the
lower bound via a polynomial-time algorithm is been long-standing open question, both for general distributions~\citep{EHKV89,BEHW89} and for the case of  the uniform distribution in the
unit ball~\citep{Lon95,Lon03,phil-doubling}.
In this work we resolve this question in the case where the underlying distribution belongs to the class of log-concave
 and nearly log-concave distributions, a wide class of distributions that includes
the gaussian distribution and  uniform distribution over any convex set,
and which has played an important role in several areas including sampling, optimization, integration, and learning~\citep{lv:2007}.

We also consider active learning, a major area of research of modern machine learning,
where the
algorithm only receives the classifications
of examples when it requests them~\citep{sanjoy11-encyc}.
Our main result here is a polynomial-time active learning algorithm
with label complexity that is exponentially better than the label complexity of any
passive learning  algorithm in these settings. This answers an open question in~\citep{BBZ07} and it also significantly expands the set of
cases for which  we can show that active learning provides a clear exponential improvement in $1/\epsilon$ (without increasing the dependence on $d$) over passive learning. Remarkably, our analysis for passive learning is done via a connection to our analysis for active learning -- to our knowledge, this is the first paper using this technique.

We also study active and passive learning in the case that the
data might not be linearly separable. We specifically provide new improved bounds for the
widely studied Tsybakov low-noise condition~\citep{MT99,BBM05,MN06}, as well as
new bounds on the disagreement coefficient, with  implications
 for the agnostic case (i.e., arbitrary forms of noise).

\smallskip
\noindent{\bf Passive Learning}~
In the classic passive supervised machine learning setting,
the learning algorithm is given a set of labeled examples drawn i.i.d.\ from some fixed but unknown
distribution over the instance space and labeled according to some fixed but unknown target function,
and the goal is to output a classifier that does well on new examples coming from the same distribution.
This setting has been long studied in both computational learning theory (within the PAC model~\citep{Valiant:acm84,KV:book94})
and statistical learning theory~\citep{Vap82,Vapnik:book98,bbl05}, and has played a crucial role in the developments and successes of machine learning.

However, despite remarkable progress, the basic question of providing
polynomial-time algorithms with {\em tight} bounds on the sample
complexity has remained open. Several milestone results along these
lines that are especially related to our work include the following.
The analysis of \citep{BEHW89}, proved using tools from~\citep{VC},
implies that linear separators can be learned in polynomial time with
$O((d/\epsilon) \log (1/\epsilon) + (1/\epsilon)\log(1/\delta))$
labeled examples.  \citep{EHKV89} proved a bound that implies an
$\Omega(d/\epsilon + (1/\epsilon)\log(1/\delta))$ lower bound for
linear separators and explicitly posed the question of providing tight
bounds for this class.  \citep{HLW94} established an upper bound of
$O((d/\epsilon) \log(1/\delta))$, which can be achieved in
polynomial-time for linear separators.

\citep{BEHW89} achieved polynomial-time learning by finding a
consistent hypothesis (i.e., a hypothesis which correctly classifies
all training examples); this is a special case of
ERM \citep{Vap82}.  An intensive line of research in the empirical
process and statistical learning theory literature has taken account
of ``local complexity'' to prove stronger bounds for ERM~        \citep{VW96,Van00,BBM05,Lon03,Men03,gine:06,Hanneke07,steve-surrogate}.
In the context of learning, local complexity takes account of the fact
that really bad classifiers can be easily discarded, and the set of
``local'' classifiers that are harder to disqualify is sometimes not
as rich.  A recent landmark result of \citep{gine:06}
(see also~\citep{RR11,steve-surrogate}) is the bound for consistent algorithms of
\begin{equation}
\label{e:cap}
O((d/\epsilon) \log
(\capacity(\epsilon)) + (1/\epsilon)\log(1/\delta))
\end{equation}
where
$\capacity(\epsilon)$ is the Alexander capacity, which depends on the
distribution~\citep{alexander87} (see Section~\ref{se:dis} and Appendix~\ref{rel-appendix} for further discussion).
However, this bound
can be
suboptimal for linear separators.

In particular, for linear separators in the case in which the
underlying distribution is uniform in the unit ball, the sample
complexity is known \citep{Lon95,Lon03} to be $\Theta\left(\frac{d
+ \log(1/\delta)}{\epsilon}\right)$, when computational considerations are ignored.  \citep{phil-doubling},
using the doubling dimension \citep{Ass83}, another measure of
local complexity, proved a bound of
\begin{equation}
\label{e:doubling}
O((d/\epsilon)
\sqrt{\log(1/\epsilon)} + (1/\epsilon)\log(1/\delta))
\end{equation}
for a polynomial-time algorithm. As a lower bound of $\Omega(\sqrt{d})$ on
$\capacity(\epsilon)$ for $\epsilon=o(1/\sqrt{d})$ for the case of linear separators
and the uniform distribution is implicit in \citep{Hanneke07},
the bound of \citep{gine:06} given by (\ref{e:cap})
cannot yield a bound better than
\begin{equation}
\label{e:cap.con}
O((d/\epsilon) \min \{ \log d, \log(1/\epsilon) \} + (1/\epsilon)\log(1/\delta))
\end{equation}
in this case.

In this paper we provide a {\em tight} bound (up to constant factors)
on the sample complexity of polynomial-time learning of linear separators with respect to log-concave distributions.
Specifically, we prove an upper bound of
$O\left(\frac{d + \log(1/\delta)}{\epsilon}\right)$
using a polynomial-time algorithm that holds for any zero-mean
log-concave distribution.
We also prove an information theoretic lower bound
that matches
our   (computationally efficient) upper bound for {\em each}
log-concave distribution.  This provides the first
bound for a polynomial-time algorithm
that is tight for an interesting non-finite class of hypothesis functions under a general class of data-distributions, and also characterizes
(up to a constant factor) the distribution-specific sample complexity for each distribution in the
class.
In the special case of the uniform distribution, our upper bound
closes the existing $\Omega(\min \{ \sqrt{\log(1/\epsilon)}, \log(d) \})$ gap between
the upper bounds (\ref{e:doubling}) and (\ref{e:cap.con}) and
the lower bound of \citep{Lon95}.

\smallskip
\noindent{\bf Active Learning}~
We also study learning of linear separators in the active learning model;
here the learning algorithm can access unlabeled
(i.e., unclassified) examples and ask for labels of
unlabeled examples of its own choice, and the hope is that a good classifier can be learned
with significantly fewer labels by actively directing the queries to informative examples. This has been a major area of machine learning research in the past fifteen years mainly due the availability of large amounts of unannotated or raw data in many modern applications~\citep{sanjoy11-encyc}, with many exciting developments on understanding its underlying principles
as well~\citep{QBC,sanjoy-coarse,BBL,BBZ07,Hanneke07,dhsm,CN07,Nina08,Kol10,nips10}. However, with a few exceptions~\citep{BBZ07,CN07,dkm}, most of the theoretical developments have focused on the so called disagreement-based active learning paradigm~\citep{hanneke:11,Kol10}; methods and analyses developed in this context are often suboptimal, as they take a conservative approach and consider strategies that query even points on which there is a small amount of uncertainty (or disagreement) among the classifiers still under consideration given the labels queried so far.
The results derived in this manner often show an improvement in the $1/\epsilon$ factor in the label complexity of active versus passive learning; however, unfortunately, the dependence on the $d$ term typically gets worse.

By analyzing a more aggressive, margin-based active learning algorithm, we prove that we can efficiently (in polynomial time) learn homogeneous linear separators when the
underlying distribution is log-concave by using only  $O((d + \log(1/\delta) + \log \log (1/\epsilon))\log (1/\epsilon))$ label requests, answering an open question
in~\citep{BBZ07}.
This represents an exponential improvement of active learning over passive learning and it significantly broadens the cases for which we can show
that the dependence on $1/\epsilon$ in passive learning can be improved to only $\tilde{O}(\log(1/\epsilon))$ in active learning, but without increasing the dependence on the dimension $d$. We note that an improvement of this type was known to be possible only for the case when the underlying distributions is (nearly) uniform  in the unit
 ball~\citep{BBZ07,dkm,QBC};
 even for this special case, our analysis
 improves by a multiplicative $\log\log(1/\epsilon)$ factor
 the results of~\citep{BBZ07}; it also provides better dependence on $d$ than any other previous analyses implementable in a computationally efficient manner (both disagreement-based~\citep{hanneke:11,Hanneke07} and more aggressive  ones~\citep{dkm,QBC}), and over the inefficient splitting index analysis of~\citep{sanjoy-coarse}.

\smallskip
\noindent{\bf Techniques}~
At the core of our results is a novel characterization of the region of disagreement of two linear separators under a log-concave measure.
We show that for any two linear separators specified by normal vectors $u$ and $v$,  for any constant $c \in (0,1)$ we can pick a margin as small as $\gamma=\theta(\alpha)$, where $\alpha$ is the angle between $u$ and $v$, and still ensure that the probability mass of the region of disagreement outside of band of margin $\gamma$ of one of them is
$c \alpha$ (Theorem~\ref{lemma:vectors-sophist}).
Using this fact, we then show how we can use a margin-based active learning technique, where in each round we only query points near the hypothesized decision boundary, to get an exponential improvement over passive learning.

We then show that any passive learning algorithm that outputs a
hypothesis consistent with $O(d/\epsilon +
(1/\epsilon)\log(1/\delta))$ random examples will, with probability at
least $1-\delta$, output a hypothesis of error at most $\epsilon$ (Theorem~\ref{t:passive}).
Interestingly,
our analysis is quite dissimilar to the classic
analyses of ERM.
It proceeds by conceptually running the
algorithm online on progressively larger chunks of examples, and using
the intermediate hypotheses to track the progress of the algorithm.
We show, using the same tools as in the active learning analysis, that
it is always likely that the algorithm will receive informative
examples.  Our analysis shows that the algorithm would also achieve
$1-\epsilon$ accuracy with high probability even if it periodically
built preliminary hypotheses using some of the examples, and then only
used borderline cases for those preliminary classifiers for
further training.\footnote{Note that such examples would not be i.i.d from the underlying distribution!}  To achieve the optimal sample complexity,
we have to carefully distribute the confidence parameter, by allowing
higher probability of failure in the later stages, to compensate for
the fact that, once the hypothesis is already pretty good, it takes
longer to get examples that help to further improve it.

\smallskip
\noindent{\bf Non-separable case}~ We also study label-efficient  learning in the presence of noise.
We show how our results for the realizable case can be extended to handle (a variant of) the
 Tsybakov noise, which has received substantial attention in statistical
 learning theory, both for passive and active learning~\citep{MT99,BBM05,MN06,gine:06,BBZ07,Kol10,hanneke:11};
 this includes the random classification noise commonly studied
 in computational learning theory~\citep{KV:book94}, and the more general bounded  (or Massart) noise~\citep{BBM05,MN06,gine:06,Kol10}.
Our analysis for Massart noise leads to optimal bounds (up to constant factors) for active and passive learning of linear separators when the marginal distribution on the feature vectors is log-concave,  improving the dependence on $d$ over previous best known results.
Our analysis for Tsybakov noise leads to bounds  on active learning with improved dependence on $d$ over previous known results in this case as well.

We also provide a bound on the Alexander's capacity~\citep{alexander87,gine:06} and the closely related disagreement coefficient notion~\citep{Hanneke07},
 which have been widely used to characterize the sample
complexity of various (active and passive) algorithms~\citep{Hanneke07,Kol10,gine:06,nips10}. This immediately
implies concrete bounds on the labeled data complexity of several
algorithms in the literature, including active learning algorithms designed for the purely agnostic case (i.e., arbitrary forms of noise), e.g.,
  the $A^2$ algorithm~\citep{BBL} and the DHM algorithm~\citep{dhsm}.

\smallskip
\noindent{\bf Nearly log-concave distributions}~
We also extend our results both for passive and active learning to
deal with nearly log-concave distributions; this is a broader class of distributions introduced by~\citep{kannan91}, which contains  mixtures of (not too separated) log-concave distributions.  In deriving our results, we provide new tail
bounds and structural results for these distributions, which might be
of independent interest and utility, both in learning theory and in other areas including sampling and optimization.

\smallskip We note that our bounds on the disagreement coefficient  improve by a factor of $\Omega(d)$  over the bounds of
   \citep{fridemnan09} (matching what was known for the much less general case of nearly uniform distribution over the unit sphere); furthermore, they apply to  the nearly log-concave case where we allow an arbitrary number of
 discontinuities, a case not captured by the \citep{fridemnan09} conditions at all.  We discuss
  other related papers 
  in Appendix~\ref{rel-appendix}.

\vspace{-2mm}
\section{Preliminaries and Notation}
\vspace{-1mm}
We focus on binary classification problems; that is, we consider the problem of predicting a binary label
$y$ based on its corresponding input vector $x$.
As in the
standard machine learning formulation, we assume  that the data
points $(x,y)$ are drawn from an unknown underlying distribution
$D_{XY}$ over $X \times \cY$; $X$ is called the {\em
instance space} and $\cY$ is the {\em label space}. In this paper
we assume that $\cY = \{\pm 1\}$  and $X=\reals^d$; we also denote
the marginal distribution over $X$ by $D$.
Let $\C$ be the class of linear separators through the
origin, that is $\C=\{ \sign(w \cdot x): w \in \reals^d, \norm{w}=1 \}$.  To keep the notation simple, we sometimes refer to a weight vector and the linear
classifier with that weight vector interchangeably.
Our goal is to output a hypothesis function $w \in \C$ of small error, where
$\err(w) = \err_{D_{XY}}(w) = P_{(x,y) \sim D_{XY}}  [\sgn(w\cdot x) \neq y].$
%\smallskip

We consider two learning protocols: passive learning and active learning. In the passive learning setting, the learning algorithm
is given a set of labeled examples $(x_1,y_1), \ldots,  (x_m,y_m)$  drawn i.i.d. from $D_{XY}$  and the goal is output a hypothesis  of small error by using only a polynomial number of labeled examples.
In the (pool-based) active learning setting, a set of labeled examples $(x_1,y_1) \ldots (x_m,y_m)$  is also drawn i.i.d.\ from $D_{XY}$;
the learning algorithm is permitted direct access to the sequence of $x_i$ values (unlabeled data points), but has to make a label request to obtain the label $y_i$ of example $x_i$.
The hope is that in the active learning setting we can output a classifier of small error by using
 many fewer label requests than in the passive learning setting  by actively directing the queries to informative examples (while keeping the number of unlabeled examples polynomial).
For added generality, we also consider the selective sampling active learning model, where the algorithm visits the
unlabeled data points $x_i$ in sequence, and, for each $i$, makes a decision
on whether or not to request the label $y_i$ based only on the
previously-observed $x_j$ values ($j \leq i$) and corresponding requested
labels, and never changes this decision once made. Both our upper and lower bounds will apply to both selective sampling and pool-based active learning.

In the ``realizable
case'', we  assume that the labels are deterministic and generated by a target function that belongs to $\C$.
In the non-realizable case (studied in Sections~\ref{se:dis} and~\ref{se:tsy}) we do not make this assumption and instead aim to compete with
the best function in $\C$.

Given two vectors $u$ and $v$ and any distribution $\tilde{D}$ we denote by
$ \distf_{\tilde{D}}(u,v)=\P_{x \sim \tilde{D}}(\sign(u \cdot x) \neq  \sign(v \cdot x))$; we also denote by $\anglesep(u,v)$ the angle between the vectors $u$ and $v$.

\vspace{-2mm}
\section{Log-Concave Densities}
\vspace{-1mm}
Throughout this paper we focus on the case where the underlying distribution $D$ is log-concave or nearly log-concave. Such distributions have played a key role in the past two decades in several areas including sampling, optimization, and integration algorithms~\citep{lv:2007}, and more recently for learning theory as well~\citep{KKMS05,KLT09,Vem10}.
In this section we first summarize known results about such distributions that are useful for our analysis and then prove a novel structural statement that will be key to our analysis (Theorem~\ref{lemma:vectors-sophist}). In Section~\ref{sec:more_distr} we describe extensions to  nearly log-concave distributions as well.

\begin{definition}
 A distribution over $R^d$ is log-concave if $\log f( \cdot )$ is concave, where $f$ is its associated density function.
 It is isotropic if its mean is the origin and its covariance matrix  is the identity.
\end{definition}

Log-concave distributions form a broad class of distributions:
for example, the Gaussian, Logistic,  
and uniform distribution over any convex set are log-concave distributions.
The following lemma summarizes known useful facts about isotropic log-concave distributions
(most are from~\citep{lv:2007}; the upper bound on the density is from
\citep{KLT09}).

\begin{lemma}
\label{isotropic-basic}
Assume that $D$  is log-concave in $R^d$ and let $f$ be its density function.
\begin{enumerate}
\vspace{-2mm}
\setlength{\itemindent}{-2mm}
\item[(a)]
If $D$ is isotropic then
$\P_{x \sim D} {[ ||X||  \geq \alpha \sqrt{d}]} \leq e^{-\alpha +1}.$
If $d=1$ then:
$\P_{x \sim D} {[ X \in [a,b]]} \leq |b-a|.$
\item[(b)] If $D$ is isotropic, then $f(x) \geq 2^{-7d}2^{−9d||x||}$ whenever $0 \leq ||x|| \leq 1/9$.
 Furthermore,
$
2^{-7d}\leq f(0) \leq d(20d)^{d/2},
$
and
$
f(x) \leq A(d) \exp(-B(d) || x ||),
$
where
$A(d)$ is $2^{8d} d^{d/2} e$ and
$B(d)$ is $\frac{2^{-7d}}{2(d-1)(20(d-1))^{(d-1)/2}}$, for all $x$
of any norm.
 \item[(c)] All marginals of $D$ are  log-concave. If $D$  is isotropic, its marginals are isotropic as well.
 \item[(d)] If $\E[\norm{X}^2]= C^2$, then  $\P {[ ||X||  \geq R C]} \leq e^{-R +1}.$
 \item[(e)] If $D$ is isotropic and $d=1$ we have $f(0) \geq 1/8$ and $f(x) \leq 1$ for all $x$.
\end{enumerate}
\end{lemma}

Throughout our paper we will use the fact that there exists a universal constant $c$ such that the probability of disagreement
of any two homogeneous linear separators  is lower bounded by the c times the angle between their normal vectors. This follows by projecting the region of disagreement  in the space given by the two normal vectors, and then using properties of log-concave distributions in 2-dimensions.
The proof is implicit in earlier works (e.g., \citep{Vem10});
for completeness, we include a proof in Appendix~\ref{a:angle}.

\newcommand{\lemmaanglebasic}
{
Assume $D$ is an isotropic log-concave in $R^d$. Then there exists $c$ such that for any two unit vectors $u$ and $v$ in $\reals^d$  we have
$ c \anglesep(v,u) \leq \distf_D(u,v).$
}

\begin{lemma}
\label{l:angle}
\label{L:ANGLE}
\lemmaanglebasic
\end{lemma}

To analyze our active and passive learning algorithms we provide a novel characterization of the region of disagreement of two linear separators under a log-concave measure:
\begin{theorem}
\label{lemma:vectors-sophist}
For any $c_1 > 0$, there is a $c_2 > 0$ such that the
following holds.
Let $u$ and $v$ be two unit vectors in $R^d$, and assume that
$\theta(u,v) = \alpha < \pi/2$. If $D$
is isotropic log-concave in $R^d$, then:
\begin{equation}
\label{e:largemargin}
\P_{x \sim D} [ \mathrm{sign}(u \cdot x) \neq \mathrm{sign}(v \cdot x)
\mbox{ and }
| v \cdot x| \geq c_2 \alpha]
 \leq c_1 \alpha.
\end{equation}
\end{theorem}

\begin{proof}
Choose $c_1,c_2 > 0$.  We will show that, if $c_2$ is large enough
relative to
$1/c_1$,
then (\ref{e:largemargin}) holds.  Let $b = c_2 \alpha$.
Let $E$ be the set whose probability we want to bound.
Since the event under consideration only concerns the projection
of $x$ onto the span of $u$ and $v$, Lemma~\ref{isotropic-basic}(c) implies
we can assume without loss of generality that $d=2$.

Next, we claim that each member $x$ of $E$ has
$|| x || \geq b/\alpha = c_2$.  Assume without loss of generality
that $v \cdot x$ is positive.  (The other case is symmetric.)  Then
$u \cdot x < 0$, so the angle of $x$ with $u$ is
obtuse, i.e. $\theta(x,u) \geq \pi/2$.
Since $\theta(u,v) = \alpha$, this implies that
$\theta(x,v) \geq \pi/2 - \alpha$.
But $x \cdot v \geq b$, and $v$ is unit length, so
$|| x || \cos \theta(x,v) \geq b$, which,
since $\theta(x,v) \geq \pi/2 - \alpha$,
implies
$
|| x || \cos (\pi/2 - \alpha) \geq b;
$
This, since
$\cos (\pi/2 - \alpha) \leq \alpha$ for all $\alpha \in [0,\pi/2]$, in turn
implies
$
|| x || \geq b/\alpha = c_2.
$
This implies that, if $B(r)$ is a ball of radius $r$ in $\R^2$, that
\begin{equation}
\label{e:shells}
\P[ E ] = \sum_{i=1}^{\infty}
                \P[ E \cap (B((i+1) c_2) - B(i c_2)) ].
\end{equation}
To obtain the desired bound, we
carefully bound each term in the RHS. Choose $i \geq 1$.

Let $f(x_1,x_2)$ be the density of $D$.  We have
\[
\P[ E \cap (B((i+1) c_2) - B(i c_2)) ] =
  \int_{(x_1,x_2) \in B((i+1) c_2) - B(i c_2)} 1_E(x_1,x_2) f(x_1,x_2) \; dx_1 dx_2.
\]
Applying the density upper bound from Lemma~\ref{isotropic-basic}
with $d=2$, there are constants $C_1$ and $C_2$ such that
\begin{align*}
\P[ E \cap (B((i+1) c_2) - B(i c_2)) ]
& \leq
  \int_{(x_1,x_2) \in B((i+1) c_2) - B(i c_2)}
     1_E(x_1,x_2) C_1 \exp(- c_2 C_2 i) \; dx_1 dx_2 \\
& =
 C_1 \exp(- c_2 C_2 i)
     \int_{(x_1,x_2) \in B((i+1) c_2) - B(i c_2)} 1_E(x_1,x_2) \; dx_1 dx_2.
\end{align*}
If we include $B(i c_2)$ in the integral again, we get
\[
\P[ E \cap (B((i+1) c_2) - B(i c_2)) ]
 \leq
 C_1 \exp(- c_2 C_2 i)  \int_{(x_1,x_2) \in B((i+1) c_2)} 1_E(x_1,x_2) \; dx_1 dx_2.
\]
Now, we exploit the fact that the integral above is a rescaling of
a probability with respect to the uniform distribution.
Let $C_3$ be the volume of the unit ball in $\R^2$.  Then, we have
\[
\P[ E \cap (B((i+1) c_2) - B(i c_2)) ]
  \leq  C_1 \exp(- c_2 C_2 i)  C_3 (i+1)^2 c_2^2 \alpha/\pi \\
 = C_4 c_2^2 \alpha (i+1)^2 \exp(-c_2 C_2 i ),
\]
for $C_4 = C_1 C_3 /\pi$.  Returning to (\ref{e:shells}), we
get
\[
\P[ E ] = \sum_{i=1}^{\infty}  C_4 c_2^2 \alpha (i+1)^2 \exp(-c_2 C_2 i)
 =  C_4  c_2^2 \times \frac{4 e^{2 c_2 C_2} - 3 e^{c_2 C_2} + 1}{\left(e^{c_2 C_2} - 1 \right)^3} \times \alpha.
\]
Since
$
\lim_{c_2 \rightarrow \infty}
c_2^2 \times \frac{4 e^{2 c_2 C_2} - 3 e^{c_2 C_2} + 1}{\left(e^{c_2 C_2} - 1\right)^3}
 = 0,
$
this completes the proof.
\end{proof}

We note that a weaker result of this type was proven (via different techniques)  for the uniform distribution in the unit ball in~\citep{BBZ07}. In addition to being more general, Theorem~\ref{lemma:vectors-sophist} is tighter and more refined even for this specific case --  this improvement is essential for obtaining tight bounds for polynomial time algorithms for passive learning (Section~\ref{passive}) and better bounds for active learning as well.

\vspace{-2mm}
\section{Active Learning}
\vspace{-1mm}
In this section we analyze a margin-based algorithm for actively learning
linear separators under log-concave distributions~\citep{BBZ07} (Algorithm~\ref{fig:active-uniform-simple-offline}).
Lower bounds proved in Section~\ref{s:lower} show that this algorithm needs
exponentially fewer labeled examples than any passive learning
algorithm.

\forspace{
Algorithm~\ref{fig:active-uniform-simple-offline} is somewhat like the ellipsoid algorithm, except (a) it
maintains a ball containing the target, instead of an ellipse, (b) it
updates this ball only after adding multiple constraints, instead of
just one, and (c) it maintains the target inside of the ball only with
high probability, instead of certainly.  Picking multiple random
constraints allows the algorithm to ensure that any new hypothesis satisfying
them is {\em closer} to the target than the old one, and therefore can
be enclosed in a smaller ball.
}
This algorithm has been previously proposed and analyzed in~\citep{BBZ07} for the special case of the uniform distribution in the unit ball.
In this paper we analyze it for the much more general class of log-concave distributions.

\begin{algorithm}[htbp]
{\bf Input}:
a sampling oracle for $\distrib$,  a labeling oracle,
sequences $m_k>0$, $k\in \integers$  (sample sizes) and
$b_k >0$, $k\in \integers$ (cut-off values).

{\bf Output}: weight vector $\hat{w}_{\rounds}$.
\begin{itemize}
\vspace{-2mm}
\setlength{\itemindent}{-4mm}
\setlength{\itemsep}{-1.5mm}
\item Draw $m_1$ examples from $\distrib$, label them and put them in $W(1)$.
\item {\bf iterate} $k=1,\ldots, \rounds$
   \begin{itemize}
   \vspace{-2mm}
\setlength{\itemindent}{-4mm}
\setlength{\itemsep}{-1.5mm}
     \item find a hypothesis $\hat{w}_k$ with $\|\hat{w}_k\|_2=1$  consistent
with all labeled examples in $W(k)$.
     \item let $W(k+1)=W(k).$
     \item  until $m_{k+1}$ additional data points are labeled,
           draw sample $x$ from $\distrib$
       \begin{itemize}
\setlength{\itemindent}{-4mm}
           \item if $|\hat{w}_{k}\cdot x| \geq b_k$, then reject $x$,
           \item else, ask for label of $x$, and put into $W(k+1)$.
       \end{itemize}
   \end{itemize}
\end{itemize}
\caption{Margin-based Active Learning}
\label{fig:active-uniform-simple-offline}
\end{algorithm}
\vspace{-2mm}

\begin{theorem}
\label{th:agg.margin}
Assume  $D$ is isotropic  log-concave in $R^d$.
There exist constants $C_1,C_2$ s.t.\ for $d
\geq 4$, and for any $\epsilon, \delta>0$, $\epsilon < 1/4$, using
Algorithm~\ref{fig:active-uniform-simple-offline} with
 $b_k=\frac{C_1}{2^{k}} $ and $m_k = C_2
\left( d + \ln\frac{1 + s - k}{\delta}\right)$,
after $s= \roundsEPSwc$
iterations, we find a separator of error at most $\epsilon$ with
probability $1-\delta$.  The total number of labeled examples
needed is
$O((d + \log(1/\delta) + \log \log (1/\epsilon))\log (1/\epsilon))$.
\end{theorem}
\begin{proof}
Let $c$ be the constant
from Lemma~\ref{l:angle}.
We will show, using induction, that, for all $k \leq s$, with
probability at least
$
1 - \frac{\delta}{2} \sum_{i < k} \frac{1}{(1 + s - i)^2},
$
any $\hat{w}$ consistent with the data in the working set $W(k)$ has $\err(\hat{w})
\leq c 2^{-k}$, so that, in particular, $\err(\hat{w_k}) \leq
c 2^{-k}$.

The case where $k=1$ follows from the standard VC bounds
(see e.g.,\citep{VC}). Assume now the claim is true for $k-1$ ($k>1$), and consider the $k$th
iteration.  Let
$S_1=\{x: |\hat{w}_{k-1}\cdot x| \leq
b_{k-1}\}$, and $S_2=\{x: |\hat{w}_{k-1}\cdot x| > b_{k-1}\}.$
By the
induction hypothesis, we know that, with probability at least
$
1 - \frac{\delta}{2} \sum_{i < k-1} \frac{1}{(1 + s - i)^2},
$
all $\hat{w}$ consistent with $W(k-1)$,
including $\hat{w}_{k-1}$, have errors at
most $c 2^{-(k-1)}$.  Consider an arbitrary such
$\hat{w}$.  By
Lemma~\ref{l:angle} we have $\theta(\hat{w}, w^*) \leq 2^{-(k-1)}$
and $\theta(\hat{w}_{k-1}, w^*) \leq 2^{-(k-1)}$, so
$\theta(\hat{w}_{k-1}, \hat{w}) \leq 4 \times 2^{-k}$.
Applying Theorem~\ref{lemma:vectors-sophist}, there is a choice of
$C_1$ (the constant such that $b_{k-1} = C_1/2^{k-1}$) that satisfies
$\Pr((\hat{w}_{k-1} \cdot x) (\hat{w} \cdot x) < 0, x \in S_2)
  \leq  \frac{c 2^{-k}}{4 }$
and
$\Pr((\hat{w}_{k-1} \cdot x) (w^* \cdot x) < 0, x \in S_2)
  \leq \frac{c 2^{-k}}{4 }.$
So
\begin{equation}
\label{e:S2.ok}
\Pr((\hat{w} \cdot x) (w^* \cdot x) < 0, x \in S_2)
  \leq  \frac{c 2^{-k}}{2}.
\end{equation}

Now let us treat the case that $x \in S_1$.
Since we are labeling $m_{k}$ data points in $S_1$ at
iteration $k-1$, classic Vapnik-Chervonenkis bounds~\citeyearpar{VC} imply that,
if $C_2$ is a large enough absolute constant, then
with probability $1 - \delta/(4(1+s-k)^2)$, for all
$\hat{w}$ consistent with the  data in $W(k)$,
\begin{equation}
\label{e:er.given.S1}
\err(\hat{w}|S_1)
 = \Pr((\hat{w} \cdot x) (w^* \cdot x) < 0 \;|\; x \in S_1)
\leq \frac{c 2^{-k}}{4  b_k} =  \frac{c}{4 C_1}.
\end{equation}

Finally, since $S_1$ consists of those points that, after projecting
onto the direction $\hat{w}_{k-1}$, fall into an interval of length
$2 b_k$, Lemma~\ref{isotropic-basic} implies that
$
\Pr(S_1) \leq 2 b_k.
$
Putting this together with (\ref{e:S2.ok}) and (\ref{e:er.given.S1}),
with probability
$1 - \frac{\delta}{2} \sum_{i < k} \frac{1}{(1 + s - i)^2}$, we have
$\err({\hat{w}}) \leq  c 2^{-k}$, completing the proof.
\end{proof}

\vspace{-2mm}
\section{Passive Learning}
\vspace{-1mm}
\label{passive}
In this section we show how an analysis that was inspired by active
learning leads to optimal (up to constant factors)  bounds for polynomial-time
algorithms for passive learning.

\newcommand{\mainpassiverealizable}{
Assume that  $D$ is zero mean and log-concave in $R^d$.
There exists an absolute constant $C_3$ s.t.\ for $d
\geq 4$, and for any $\epsilon, \delta>0$, $\epsilon < 1/4$,
any algorithm that outputs a hypothesis that correctly
classifies
$m = \frac{C_3 (d + \log(1/\delta))}{\epsilon}$ examples
finds a separator of error at
most $\epsilon$ with probability $\geq 1-\delta$.
}

\begin{theorem}
\label{t:passive}
\mainpassiverealizable
\end{theorem}
\begin{sketch}
We focus here on the case that $D$ is
isotropic. We can treat the non-isotropic case by
observing that the two cases are equivalent; one may pass between them by applying the whitening transform. (See Appendix~\ref{appendix:mainpassiverealizable} for details.)

While our analysis will ultimately provide a guarantee for any
learning algorithm that always outputs a consistent hypothesis, we
will use intermediate hypothesis of
Algorithm~\ref{fig:active-uniform-simple-offline} in the analysis.

Let $c$ be the constant from Lemma~\ref{l:angle}.  While
proving Theorem~\ref{th:agg.margin}, we proved that,
if Algorithm~\ref{fig:active-uniform-simple-offline} is run with
$b_k=\frac{C_1}{2^{k}} $ and $m_k = C_2 \left( d + \ln\frac{1 + s -
  k}{\delta}\right)$, that for all $k \leq s$, with probability
$
\geq 1 - \frac{\delta}{2} \sum_{i < k} \frac{1}{(1 + s - i)^2}
$
any
$\hat{w}$ consistent with the data in $W(k)$ has
$\err(\hat{w}) \leq c 2^{-k}$.
Thus, after $s = O(\log(1/\epsilon))$ iterations, with probability
at
least
$\geq 1 - \delta$, any linear classifier consistent with {\em all} the
training data has error $\leq \epsilon$, since any such classifier
is consistent with the examples in $W(s)$.

Now, let us analyze the number of examples used, including those
examples whose labels were not requested by
Algorithm~\ref{fig:active-uniform-simple-offline}.
Lemma~\ref{isotropic-basic} implies that there is a positive constant
$c_1$ such that $\P(S_1) \geq c_1 b_k$: again, $S_1$ consists of those
points that fall into an interval of length $2 b_k$ after projecting onto
$\hat{w}_{k-1}$.  The density is lower bounded by a constant when
$b_k \leq 1/9$, and we can use the bound for $1/9$ when $b_k > 1/9$.
The expected number of examples that we need before we
find $m_k$ elements of $S_1$ is therefore at most $\frac{m_k}{c_1 b_k}$.
Using a Chernoff bound, if we draw $\frac{2 m_k}{c_1 b_k}$
examples, the probability that we fail to get $m_k$ members of $S_1$
is at most $\exp\left(- m_k/6\right)$, which is
at most $\delta/(4(1+s-k)^2)$ if $C_2$ is large enough.
So, the total number of examples needed, $\sum_k \frac{2 m_k}{c_1 b_k}$,
is at most a constant factor more than
\begin{align*}
& \sum_{k=1}^s 2^k \left( d
             + \log \left( \frac{1 + s - k}{\delta}\right) \right)
  = O(2^s (d + \log(1/\delta))) + \sum_{k=1}^s 2^k \log (1 + s - k) \\
& = O\left(\frac{d + \log(1/\delta)}{\epsilon}\right)
            + \sum_{k=1}^s 2^k \log (1 + s - k).
\end{align*}
We can show $\sum_{k=1}^s 2^k \log (1 + s - k) =O(1/\epsilon)$,
completing the proof.
\end{sketch}

\smallskip

We conclude this section by pointing out several important facts and implications of Theorem~\ref{t:passive} and its proof.
\begin{enumerate}
\item[(1)] The separator in Theorem~\ref{t:passive} (and the one in Theorem~\ref{th:agg.margin} ) can be found in {\em polynomial
time}, for example by using linear programming. 
\item[(2)] The analysis of Theorem~\ref{t:passive} also bounds
the number of unlabeled examples needed by the active learning algorithm
of Theorem~\ref{th:agg.margin}.  This shows that an algorithm can
request a nearly optimally small number of labels without increasing the total
number of examples required by more than a constant factor. Specifically, in round $k$, we only need $2^k (d + \ln[(1+s -k)/\delta])$
unlabeled examples (whp), where $s= O(\log(1/\epsilon))$, so
the total number of unlabeled examples needed  over all rounds is $O(d/\epsilon +\log(1/\delta)/\epsilon)$.
\end{enumerate}

\vspace{-2mm}
\section{More Distributions}
\label{sec:more_distr}
\vspace{-1mm}
In this section we consider learning with respect to a more general
class of distributions. We start by providing a general set of
 conditions on a set $\D$ of distributions
that is sufficient for
efficient passive and active learning w.r.t.\ distributions
in $\D$. We now consider  nearly log-concave distributions, an interesting, more general class containing log-concave distributions, considered
previously in~\citep{kannan91} and~\citep{shie07}.  We then  prove that isotropic nearly log-concave distributions satisfy our sufficient conditions;
in Appendix~\ref{app:more_distr}, we also show how to remove the
assumption that the distribution is isotropic.

\begin{definition}
A set $\cal D$ of distributions is {\em admissible} if it satisfies
the following:
\begin{itemize}
\vspace{-2mm}
\setlength{\itemindent}{-3.5mm}
\setlength{\itemsep}{-1.5mm}
\item There exists $c$ such that for any $D \in {\cal D}$ and any
two unit vectors $u$ and $v$ in $\reals^d$  we have
$ c \anglesep(v,u) \leq \distf_D(u,v).$
\item   For any $c_1 > 0$, there is a $c_2 > 0$ such that the
following holds for all $D \in {\cal D}$.
Let $u$ and $v$ be two unit vectors in $R^d$ s.t.
$\theta(u,v) = \alpha < \pi/2$.  Then
$
\P_{x \sim D} [ \mathrm{sign}(u \cdot x) \neq \mathrm{sign}(v \cdot x),
| v \cdot x| \geq c_2 \alpha]
 \leq c_1 \alpha.
$
\item   There are positive constants $c_3, c_4, c_5$ such that,
for any $D' \in {\cal D}$, for any projection $D$ of
$D'$ onto a one-dimensional subspace, the density $f$ of
$D$ satisfies $f(x) < c_3$ for all $x$ and $f(x) > c_4$ for all $x$ with $|x| < c_5$.
\end{itemize}
\end{definition}

The proofs of Theorem~\ref{th:agg.margin}
and Theorem~\ref{t:passive} can be used without modification to show:
\begin{theorem}
\label{t:admissible}
If $\D$ is admissible, then arbitrary
$f \in \C$ can be learned with respect to arbitrary distributions in
$\D$ in polynomial time in the active learning model from
$O((d + \log(1/\delta) + \log \log (1/\epsilon))\log (1/\epsilon))$
labeled examples, and in the passive learning model from
$O\left(\frac{d + \log(1/\delta)}{\epsilon}\right)$ examples.
\end{theorem}

\subsection{The nearly log-concave case}
\label{se:near-log}
\begin{definition}
A density function $f: \reals^{n} \rightarrow \reals$ is $\beta$ log-concave if for any $\lambda \in [0,1]$, $x_1 \in \reals^n$, $x_2 \in \reals^n$,
we have $f(\lambda x_1 +(1-\lambda) x_2)\geq e^{-\beta} f(x_1)^\lambda f(x_2)^{1-\lambda}$.
\end{definition}

Clearly, a density function $f$ is log-concave if it is $0$-log-concave.
An example of a $O(1)$-log-concave distribution
is a mixture of two log-concave distributions whose covariance matrices
are $I$, and whose means $\mu_1$ and $\mu_2$ have
$|| \mu_1 - \mu_2 || = O(1)$.

In this section we prove that for any  sufficiently small constant $\beta \geq 0$,
the class of isotropic $\beta$ log-concave distribution in $R^d$ is admissible and has light tails (this second fact is useful for analyzing the disagreement coefficient in  Sections~\ref{se:dis}).
In doing so we provide several new properties for such distributions, which could be of independent interest. Detailed proofs of our claims appear in Appendix~\ref{app:more_distr}.

We start by  showing that for any isotropic $\beta$ log-concave density $f$ there exists a log-concave density  $\tilde{f}$  whose center is within $e(C-1) \sqrt{Cd}$ of $f$'s center and that satisfies
$f(x)/C \leq \tilde{f}(x) \leq Cf(x)$,  for  $C$ as small as $e^{\beta \log d }$. The fact $C$ depends only exponentially in $\log d$ (as opposed to exponentially in $d$) is key for being able to argue that such distributions have light tails.

\newcommand{\lemmabasicbetalc}
{
For any isotropic $\beta$ log-concave density function $f$ there exists a log-concave density function $\tilde{f}$  that satisfies
$f(x)/C \leq \tilde{f}(x) \leq Cf(x)$ and $\norm{\int {x(f(x) -\tilde{f}(x))dx}} \leq e(C-1) \sqrt{Cd}$, for  $C=e^{\beta \lceil \log_2 (d+1) \rceil}$.
Moreover,  we have 
$1/C \leq \int {(u \cdot x)^2 \tilde{f}(x) dx} \leq C$ for every unit vector $u$.
}

\begin{lemma}
\label{basic-betalc}
\lemmabasicbetalc
\end{lemma}
\begin{sketch}
Note that if the density function $f$ is $\beta$ log-concave we have
that $h=\ln f$ satisfies that for any $\lambda \in [0,1]$, $x_1 \in
\reals^n$, $x_2 \in \reals^n$, we have $h(\lambda x_1 +(1-\lambda)
x_2)\geq -\beta + \lambda h(x_1)+ (1-\lambda) h (x_2)$.
Let $\hat{h}$ be the function whose subgraph is the convex hull of
the subgraph of $h$.
By using Caratheodory's
theorem\footnote{Caratheodory's theorem states that if a point $x$ of $R^d$ lies in the convex hull
of a set $P$, then there is a subset $\hat{P}$ of $P$ consisting of $d+1$ or fewer points such that $x$
lies in the convex hull of $\hat{P}$.}
we can show that
$
\hat{h}(x)=\max_{\sum_{i=1}^{d+1}{\alpha_i=1, \alpha_i\geq 0, x=\sum_{i=1}^{d+1} \alpha_i x_i}}{\sum_{i=1}^{d+1}\alpha_i h(x_i) }.
$
This implies $h(x) \leq \hat{h}(x)$  and we can prove by induction on $\log_2 (d + 1)$ that
$h(x) \geq \hat{h}(x) - \beta \lceil \log_2 (d + 1) \rceil.$
If we further normalize $e^{\hat{h}}$ to make it a
density function, we obtain $\tilde{f}$ that is log-concave and
satisfies $f(x)/C \leq \tilde{f}(x) \leq Cf(x),$ where $C=e^{\beta
  \lceil \log_2 (d+1)\rceil}.$ This implies that for any $x$ we have
$|f(x)-\tilde{f}(x)| \leq (C-1)\tilde{f}(x)$.

Using this fact and concentration properties of $\tilde{f}$ (in particular Lemma~\ref{isotropic-basic}), we can show that
 the center of $\tilde{f}$ is close to the center of $f$, as desired.
\end{sketch}

\begin{theorem}
\label{l:angle.beta} Assume $\beta$ is a sufficiently
small non-negative constant and let $\cal D$ be the set of all isotropic $\beta$ log-concave distributions.
(a) $\cal D$ is admissible.
(b) Any $D \in {\cal D}$ has light tails.
That is: $\P(||X|| > R \sqrt{C d}) < C e^{-R+1}$, for
$C=e^{\beta \lceil \log_2 (d+1) \rceil}$.
\end{theorem}
\begin{sketch}
(a) Choose $D \in {\cal D}$.
As in Lemma~\ref{l:angle}, consider the plane determined by $u$ and $v$ and let
$proj_{u,v}(x)$ denote the projection operator that
given $x \in R^d$,
orthogonally projects $x$ onto this plane.
If $D_2 = proj_{u,v}(D)$ then $\distf_D(u,v)= \distf_{D_2} (u',v').$
By using the Prekopa-Leindler inequality~\citep{Gar02} one can show that $D_2$ is  $\beta$ log-concave
 (see e.g.,~\citep{shie07}). Moreover, if $D$ is isotropic, than $D_2$  is isotropic as well.
 By Lemma~\ref{basic-betalc}  we know that 
 there exists a $C$-isotropic log-concave distribution $\tilde{D_2}$
  centered at $z$, $\norm{z} \leq \epsilon$, satisfying $f(x)/C \leq \tilde{f}(x) \leq Cf(x)$ and $1/C \leq \int {(u \cdot x)^2 f(x) dx} \leq C$ for every unit vector $u$,
  for constants
 $C=e^{\beta}$ and $\epsilon=e(C-1)\sqrt{2C}$. For $\beta$ sufficiently small 
we have $(1/20+\epsilon)/ \sqrt{1/C-\epsilon^2} \leq 1/9$. Using this,  by applying 
the whitening transform (see Theorem~\ref{almost isotropic}
 in Appendix~\ref{app:more_distr}), we can show $\tilde{f_2}(x) \geq c$, for $\norm{x} \leq 1/20$ ,   which implies $f_2(x) \geq c/C$, for $\norm{x} \leq 1/20$. Using a reasoning as in Lemma~\ref{l:angle}  we get $ c \anglesep(v,u) \leq \distf_D(u,v).$
The generalization of
Theorem~\ref{lemma:vectors-sophist} follows from a
similar proof, except using Theorem~\ref{almost isotropic}.
The density bounds in the $n=1$ case also follow from Theorem~\ref{almost isotropic} as well.

(b) Since $X$ is isotropic, we have $\E_f[ X \cdot X] = d$ (where $f$ is its associated density). By Lemma~\ref{basic-betalc}, there exists
a log-concave density $\tilde{f}$ such that $f(x)/C \leq \tilde{f}(x) \leq Cf(x)$, for $C=e^{\beta \lceil \log_2 (d+1) \rceil}$.
This implies $E_{\tilde{f}}[X \cdot X] \leq Cd$.
By Lemma~\ref{isotropic-basic} we get that that under
  $\tilde{f}$, $\P(||X|| > R \sqrt{C d}) < e^{-R+1}$, so under $f$ we have $\P(||X|| > R \sqrt{C d}) < C e^{-R+1}$.
\end{sketch}

Using  Theorem~\ref{t:admissible} and Theorem~\ref{l:angle.beta}(a) we obtain:
\begin{theorem}
\label{th:agg.margin:beta}
Let  $\beta \geq 0$  be a sufficiently small constant. Assume that $D$ is an isotropic $\beta$ log-concave distribution in $R^d$.
Then arbitrary
$f \in \C$ can be learned with respect to $D$
in polynomial time in the active learning model from
$O((d + \log(1/\delta) + \log \log (1/\epsilon))\log (1/\epsilon))$
labeled examples, and in the passive learning model from
$O\left(\frac{d + \log(1/\delta)}{\epsilon}\right)$ examples.
\end{theorem}

\vspace{-2mm}
\section{Lower Bounds}
\label{s:lower}
\vspace{-1mm}
In this section we give lower bounds on the label complexity of
passive and active learning of homogeneous linear separators when the
underlying distribution is $\beta$
log-concave, for a sufficiently small constant $\beta$.
These lower bounds are information theoretic, applying to any procedure, that might not be necessarily computationally efficient.  The proof is
in Appendix~\ref{a:lower}.

\newcommand{\thmtlower}{
For a small enough constant $\beta$ we have: (1) for any
$\beta$ log-concave distribution $D$ whose covariance matrix has
full rank, the sample complexity of learning origin-centered
linear separators under $D$ in the passive learning model is
$
\Omega\left(\frac{d}{\epsilon} + \frac{1}{\epsilon} \log\left( \frac{1}{\delta}\right)\right);
$
(2) the sample complexity of active  learning of
linear separators under
$\beta$ log-concave distributions is
$
\Omega\left(d \log\left( \frac{1}{\epsilon}\right)\right).
$

}

\begin{theorem}
\label{t:lower}
\thmtlower
\end{theorem}

Note that, if the covariance matrix of $D$ does not have full rank,
the number of dimensions is effectively less than $d$, so our lower
bound essentially applies for all log-concave distributions.

\vspace{-3mm}
\section{The inseparable case: Disagreement-based active learning}
\vspace{-1mm}
\label{se:dis}
We consider two closely related distribution dependent capacity notions: the Alexander capacity and the disagreement coefficient; they have been widely used for analyzing the label complexity of non-aggressive active learning algorithms~\citep{Hanneke07,dhsm,Kol10,hanneke:11,nips10}.
We begin with the definitions.
For $r > 0$, define $B(w,r) = \{u \in \C : \P_{D}(\sign(u \cdot x) \neq \sign(w \cdot x)) \leq r\}$.
For any $\H \subseteq \C$, define the region of disagreement as
$\DIS(\H) = \{x \in X : \exists w, u \in \H \text{ s.t. } \sign(u \cdot x) \neq \sign(w \cdot x))\}.$
Define the Alexander capacity function $\capacity_{w^{*},D}( \cdot )$ for $w^{*} \in \C$ w.r.t. $D$ as:
$\capacity_{w^{*},D}(r) = \frac{ \P_D(\DIS(B(w^{*},r))) }{ r}.$
Define the disagreement coefficients for $w^{*} \in \C$ w.r.t. $D$ as:
$ \dis_{w^*,D}(\epsilon) = \sup \limits_{r \geq \epsilon} [\capacity_{w^{*},D}(r)].$

The following is our bound in the disagreement coefficient.
Its proof is in Appendix~\ref{a:dis-coeff:beta}.

\newcommand{\thdisagreement}{
Let  $\beta \geq 0$  be a sufficiently small constant. Assume that $D$ is an isotropic $\beta$ log-concave distribution in $R^d$.
For any $w^{*}$, for any $\epsilon$, $\capacity_{w^*,D}(\epsilon)$ is
$O(d^{1/2 + \frac{\beta}{2 \ln 2}} \log(1/\epsilon))$. Thus $\dis_{w^*,D}(\epsilon)=O(d^{1/2 + \frac{\beta}{2 \ln 2}} \log(1/\epsilon))$.}

\begin{theorem}
\label{th:dis-coeff:beta}
\thdisagreement
\end{theorem}

Theorem~\ref{th:dis-coeff:beta} immediately
leads to concrete bounds on the label complexity of several
algorithms in the literature~\citep{Hanneke07,CAL,BBL,Kol10,dhsm}.
For example, by composing it with a result of~\citep{dhsm},
 we obtain a bound of $\tilde{O}(d^{3/2} (\log^2 (1/\epsilon) + (\nu/\epsilon)^2))$ for
 agnostic active learning when $D$  is isotropic log-concave in $R^d$; that is we only need  $\tilde{O}(d^{3/2} (\log^2 (1/\epsilon) + (\nu/\epsilon)^2)))$ label requests to output a classifier of error at most $\nu+\epsilon$, where $\nu=\min_{w \in \C} \err(w)$.

\vspace{-3mm}
\section{The Tsybakov condition}
\label{se:tsy}
\vspace{-1mm}
In this section we consider a variant of the Tsybakov noise condition \citep{MT99}.
We assume that the classifier
$h$ that minimizes $\P_{(x,y) \sim D_{XY}}(h(x) \neq y)$
is a linear classifier, and that, for the weight vector $w^*$ of the
optimal classifier,
there exist known parameters $\alpha, \betats>0$ such that,
for all $w$, we have
$$
 \betats (\distf_{D}(w,w^*))^{1/(1-\alpha)} \leq
\err(w)-\err(w^*).
 $$
By generalizing Theorem~\ref{lemma:vectors-sophist} so that it provides
a stronger bound for larger margins, and combining the result
with the other lemmas of this paper and techniques from \citep{BBZ07},
we get the following.

\newcommand{\thmtsybakov}{
Let $s=O(\log(1/\epsilon))$. Assume that the 
distribution
$D_{XY}$ satisfies the
Tsybakov noise condition for constants
$\alpha \in [0,1)$ and $\betats \geq 0$, and that the marginal
$D$ on $\R^d$ is isotropic  log-concave.
(1) If $\alpha = 0$,
we can find a separator with excess error
$\leq \epsilon$  with probability $1-\delta$
using $O(\log (1/\epsilon)) (d + \log(s/\delta)) $
labeled examples in the active learning model,
and $O\left(\frac{d + \log(1/\delta)}{\epsilon}\right)$ labeled examples in the passive learning model.
(2) If $\alpha > 0$,
we can find a separator with excess error
$\leq \epsilon$  with probability $1-\delta$
using $O((1/\epsilon)^{2 \alpha}  \log^2 (1/\epsilon)) (d + \log(s/\delta)) $
labeled examples in the active learning model.
}

\begin{theorem}
\label{th:agg.margin:tsybakov}
\thmtsybakov
\end{theorem}

\noindent
In the case $\alpha = 0$ (that is more general than the 
Massart noise condition) 
our analysis leads to optimal bounds 
for active and passive learning of linear separators under
 log-concave distributions,
improving the dependence on $d$ over previous best known results~\citep{steve-surrogate,gine:06}.
Our analysis for  Tsybakov noise ($\alpha \geq 0$) leads to bounds  on active learning with improved dependence on $d$ over previous known results~\citep{steve-surrogate} in this case as well. Proofs and further details appear in Appendix~\ref{a:massart}.

\section{Discussion and Open Questions}
The label sample complexity of our active learning algorithm for learning homogeneous linear separators under
 isotropic  logconcave distributions is  $O((d + \log(1/\delta) + \log \log (1/\epsilon))\log (1/\epsilon))$, while
our lower bound for this setting is $\Omega\left(d \log\left( \frac{1}{\epsilon}\right)\right).$
Our upper bound is achieved by an algorithm that uses a polynomial number
of unlabeled training examples, and polynomial time.  If an
unbounded amount of computation time and an unbounded
number of unlabeled examples are available, it seems to be easy
to learn to accuracy $\epsilon$ using $O(d \log (1/\epsilon))$
label requests, no matter what the value of $\delta$.
(Roughly, the algorithm can construct an $\epsilon$-cover to
initialize a set of candidate hypotheses, then repeatedly wait
for an unlabeled example that evenly splits the current list
of candidates, and ask its label, eliminated roughly half of
the candidates.)
It would be interesting to know what is the best label complexity
for a polynomial-time algorithm, or even an algorithm that
is constrained to use a polynomial number of unlabeled examples.

Conceptually, our analysis of ERM for passive learning under (nearly) log-concave distributions is based on a more aggressive localization 
than those considered previously in the literature. It would  be very interesting to extend this analysis as well as our
analysis for active learning to arbitrary distributions and more general concept spaces.

\paragraph{Acknowledgements}
We thank Steve Hanneke for a number of useful discussions.

\noindent This work was supported in part by NSF grant CCF-0953192, AFOSR grant FA9550-09-1-0538, and a Microsoft
Research Faculty Fellowship.

% \bibliography{active-colt}

\appendix

\section{Additional Related Work}
\label{rel-appendix}
\smallskip
\noindent{\bf Learning with noise. Alexander Capacity and the Disagreement Coefficient}~
Roughly speaking the Alexander capacity~\citep{alexander87,gine:06} quantifies how fast
the region of disagreement of the set of classifiers at distance $r$ of the optimal classifier collapses as a function $r$;
\footnote{The region of disagreement $\DIS(\C)$ of a set of classifiers $\C$ is the of set of instances $x$ s.t. for each  $x \in \DIS(\C)$ there exist  two classifiers $f,g \in \C$ that disagree about the label of $x$.}
 the disagreement coefficient~\citep{Hanneke07}
additionally involves the supremum of $r$ over a range of values.
\citep{fridemnan09} provides guarantees on these quantities  (for sufficiently small $r$) for general classes of functions in $\reals^d$ if the underlying data distribution
is sufficiently smooth. Our analysis implies much tighter bounds
for linear separators under log-concave distributions (matching what was known for the much less general case of nearly uniform distribution over the unit sphere); furthermore,
 we also analyze the nearly log-concave case where we allow an arbitrary number of
 discontinuities, a case not captured by the \citep{fridemnan09} conditions at all.
This immediately
implies concrete bounds on the labeled data complexity of several
algorithms in the literature including the $A^2$ algorithm~\citep{BBL} and the DHM algorithm~\citep{dhsm}, with implications for the purely agnostic case (i.e., arbitrary forms of noise), as well as the
Koltchinskii's algorithm~\citep{Kol10} and the CAL algorithm~\citep{BBL,Hanneke07,hanneke:11}.
Furthermore, in the realizable case and under Tsybakov noise, we show even better bounds, by considering aggressive active learning algorithms.

Note that as opposed to the realizable case, all existing active learning algorithms analyzed under Massart and Tsybakov noise
conditions
using the learning model analyzed in this paper
(including our algorithms in Theorem~\ref{th:agg.margin:tsybakov}),
as well as  those for the agnostic setting, are not known to run in
time $\poly(d, 1/\epsilon)$. In fact, even ignoring the optimality of sample complexity,
 there are no known algorithms for passive learning that run in time $\poly(d, 1/\epsilon)$ for general
values of $\epsilon$, even for the Massart noise condition and under log-concave distributions.
Existing works on agnostic passive learning  under log-concave distributions either provide running
times $d^{\poly(1/\epsilon)}$ (e.g., the work of~\citep{KKMS05})
or can only achieve values of $\epsilon$ that are significantly larger than the noise rate~\citep{kls09}.

\smallskip
\noindent{\bf Other Work on Active Learning}~
Several papers~\citep{CaCEGe10,dgs12} present efficient online learning algorithms in the selective sampling framework, where labels must
be actively queried before they are revealed. Under the assumption that the label conditional distribution is linear function determined by a fixed target vector,
they provide bounds on the regret of the algorithm and on
the number of labels it queries when faced with an adaptive adversarial strategy of generating the
instances. As pointed by ~\citep{dgs12}, these results can be converted to a statistical setting when the instances $x_t$ are drawn i.i.d and
they further assume
a margin condition. In this setting they obtain exponential improvement in label complexity over passive learning.
 While very interesting, these results are incomparable to ours;
 their techniques significantly exploit the linear noise condition to get these improvements -- note that such an improvement
 would not be possible in the realizable case (as pointed for example in~\citep{GSS12}).

~\citep{Now11} considers an interesting abstract ``generalized binary
search'' problem with applications to active learning; while these results apply for more general concept spaces, it is not clear
how to implement the resulting procedures in  polynomial time and by using access to only a polynomial number of unlabeled samples
from the underlying distribution (as required by the active learning model).
Another interesting recent work is that of~\citep{GSS12}, which study active learning of linear separators via an
 aggressive algorithm using a margin
condition, using a general approximation guarantee on the
number of labels requested; note that while these results work for
potentially more general distributions, as opposed to ours, they do not come with explicit (tight) bounds on the label complexity.

\smallskip

\noindent{\bf $\epsilon$-nets, Learning, and Geometry}~
Small $\epsilon$-nets are useful for many applications, especially in
Computational Geometry (see \citep{PA95}).  The same
fundamental techniques of \citep{VC,Vap82} have
been applied to establish the existence of small $\epsilon$-nets
\citep{HW87} and
to bound the sample complexity of learning~\citep{Vap82,BEHW89}, and a
number of interesting upper and lower bounds on the smallest possible
size of $\epsilon$-nets have been obtained \citep{KPW92,CV07,Alo10}.

Our analysis implies a $O(d/\epsilon)$ upper bound on the
size of an $\epsilon$-net for a set of regions of disagreement
between all possible linear classifiers and the target, when the
distribution is zero-mean and log-concave.
In particular, since in Theorem~\ref{t:passive} we prove that any hypothesis
consistent with the training data has error rate  $\leq \epsilon$ with
probability $1 - \delta$, setting $\delta$ to a constant gives a proof
of a $O(d/\epsilon)$ bound on the
size of an $\epsilon$-net for the following set:
$
\{ \{ x : (w \cdot x) (w^* \cdot x) < 0 \} \;:\; w \in \R^n \}.
$

\section{Proof of Lemma~\protect\ref{l:angle}}
\label{a:angle}
\repeatclaim{Lemma~\ref{l:angle}}{\lemmaanglebasic}
\begin{proof}
Consider two unit vectors $u$ and $v$. Let
$proj_{u,v}(x)$ denote the projection operator that,
given $x \in R^d$,
orthogonally projects $x$ onto the plane determined by $u$ and $v$.  That is, if we define an
orthogonal coordinate system in which coordinates $1,2$ lie in this
plane and coordinates $3,\ldots,d$ are orthogonal to this plane, then
$x'=proj_{u,v}(x_1,\ldots, x_d) = (x_1,x_2)$.  Also, given
distribution $D$ over $R^d$, define $proj_{u,v}(D)$ to be the
distribution given by first picking $x \sim D$ and then outputting
$x' = proj_{u,v}(x)$.  That is, $proj_{u,v}(D)$ is just the marginal
distribution over coordinates $1,2$ in the above coordinate system.
Notice that if $x' = proj_{u,v}(x)$ then $u \cdot x = u' \cdot x'$  where $u' = proj_{u,v}(u)$ and $v' = proj_{u,v}(v)$. So, if $D_2 = proj_{u,v}(D)$ then
$\distf_D(u,v)= \distf_{D_2} (u',v').
$

By  Lemma~\ref{isotropic-basic}(c), we have that if $D$ is
isotropic and log-concave, then $D_2$ is as well.
Let
$A$ to be the region of disagreement between $u'$ and $v'$ intersected
with the ball of radius $1/9$ in $R^2$. The probability mass of $A$
under $D_2$ is at least the volume of $A$ times $\inf_{x \in A}
D_2(x)$.  So, using Lemma~\ref{isotropic-basic}(b)
$$ \distf_{D_2} (u',v')  \geq
\vol(A)  \inf_{x \in A} D_2(x)  \geq   c \anglesep(u,v),
$$ as desired.
\end{proof}

\section{Passive Learning}
\label{appendix:mainpassiverealizable}

\repeatclaim{Theorem~\ref{t:passive}}{\mainpassiverealizable}
\begin{proof}
First, let us prove the theorem in the case that $D$ is
isotropic.  We will then treat the general case at the end
of the proof.

While our analysis will ultimately provide a guarantee for any
learning algorithm that always outputs a consistent hypothesis, we
will use intermediate hypothesis of
Algorithm~\ref{fig:active-uniform-simple-offline} in the analysis.

Let $c$ be the constant from Lemma~\ref{l:angle}.  While
proving Theorem~\ref{th:agg.margin}, we proved that,
if Algorithm~\ref{fig:active-uniform-simple-offline} is run with
$b_k=\frac{C_1}{2^{k}} $ and $m_k = C_2 \left( d + \ln\frac{1 + s -
  k}{\delta}\right)$, that for all $k \leq s$, with probability
$
\geq 1 - \frac{\delta}{2} \sum_{i < k} \frac{1}{(1 + s - i)^2}
$
any
$\hat{w}$ consistent with the data in $W(k)$ has
$\err(\hat{w}) \leq c 2^{-k}$.
Thus, after $s = O(\log(1/\epsilon))$ iterations, with probability
at
least
$\geq 1 - \delta$, any linear classifier consistent with {\em all} the
training data has error $\leq \epsilon$, since any such classifier
is consistent with the examples in $W(s)$.

Now, let us analyze the number of examples used, including those
examples whose labels were not requested by
Algorithm~\ref{fig:active-uniform-simple-offline}.
Lemma~\ref{isotropic-basic} implies that there is a positive constant
$c_1$ such that $\P(S_1) \geq c_1 b_k$: again, $S_1$ consists of those
points that fall into an interval of length $2 b_k$ after projecting onto
$\hat{w}_{k-1}$.  The density is lower bounded by a constant when
$b_k \leq 1/9$, and we can use the bound for $1/9$ when $b_k > 1/9$.

The expected number of examples that we need before we
find $m_k$ elements of $S_1$ is therefore at most $\frac{m_k}{c_1 b_k}$.
Using a Chernoff bound, if we draw $\frac{2 m_k}{c_1 b_k}$
examples, the probability that we fail to get $m_k$ members of $S_1$
is at most $\exp\left(- m_k/6\right)$, which is
at most $\delta/(4(1+s-k)^2)$ if $C_2$ is large enough.
So, the total number of examples needed, $\sum_k \frac{2 m_k}{c_1 b_k}$,
is at most a constant factor more than
\begin{eqnarray*}
&& \sum_{k=1}^s 2^k \left( d
             + \log \left( \frac{1 + s - k}{\delta}\right) \right) \\
&& = O(2^s (d + \log(1/\delta))) + \sum_{k=1}^s 2^k \log (1 + s - k) \\
&& = O\left(\frac{d + \log(1/\delta)}{\epsilon}\right)
            + \sum_{k=1}^s 2^k \log (1 + s - k). \\
\end{eqnarray*}
We claim that $\sum_{k=1}^s 2^k \log (1 + s - k) = O(1/\epsilon)$.
We have
\begin{eqnarray*}
&& \sum_{k=1}^s 2^k \log (1 + s - k) \leq \sum_{k=1}^s 2^k (3 + s - k) \\
&& \leq \int_{k=1}^{s+1} 2^k (3 + s - k) \\
&& \hspace{0.5in}
 \mbox{(since $2^k (3 + s - k)$ is increasing for $k \leq s+1$)} \\
&& = \frac{2 (2^s - 1)(1 + \ln (4)) - s \ln 2}{\ln^2 2} = O(1/\epsilon),
\end{eqnarray*}
completing the proof in the case that $D$ is
isotropic.

Now let us treat the case in which $D$ is not isotropic.
Suppose that $\Sigma$ is the covariance matrix of $D$, so that
$\Sigma^{-1/2}$ is the ``whitening transform''.  Suppose, for $m =
\frac{C_3 (d + \log(1/\delta))}{\epsilon}$, an algorithm is given a
sample $S$ of examples $(x_1,y_1),...,(x_m,y_m)$ for $x_1,...,x_m$
drawn according to $D$, and $y_m$ labeled by a target hypothesis
with weight vector $v$.  Note that $w$ is consistent with $S$ if and
only if $w^T \Sigma^{1/2}$ is consistent with $(\Sigma^{-1/2}
x_1,y_1),...,(\Sigma^{-1/2} x_m,y_m)$ (so those examples
are consistent with $v^T \Sigma^{1/2}$).
So our analysis of the isotropic
case implies that, with probability $1 - \delta$, for any $w$
consistent with $(x_1,y_1),...,(x_m,y_m)$, we have
\[
\Pr(\sign((w^T \Sigma^{1/2}) (\Sigma^{-1/2} x))
\neq \sign((v^T \Sigma^{1/2}) (\Sigma^{-1/2} x))) \leq \epsilon,
\]
which of course means that
$
\Pr(\sign(w^T x) \neq \sign(v^T x)) \leq \epsilon.
$
\end{proof}

\section{More Distributions}
\label{app:more_distr}

\subsection{Isotropic Nearly Log-concave distributions}

\repeatclaim{Lemma~\ref{basic-betalc}}{\lemmabasicbetalc}
\begin{proof}
Note that if the density function $f$ is $\beta$ log-concave we have
that $h=\ln f$ satisfies that for any $\lambda \in [0,1]$, $x_1 \in
\reals^n$, $x_2 \in \reals^n$, we have $h(\lambda x_1 +(1-\lambda)
x_2)\geq -\beta + \lambda h(x_1)+ (1-\lambda) h (x_2)$.

Let $\hat{h}$ be the function whose subgraph is the convex hull of
the subgraph of $h$.  That is,
$\hat{h}(x)$ is the maximum of all values of
$\sum_{i=1}^k \alpha_i h(u_i)$ for any $u_1,...,u_k \in \R^d$ and
$\alpha_1,...,\alpha_k \in [0,1]$ such that
$\sum_{i=1}^k \alpha_i = 1$ and $x=\sum_{i=1}^{k} \alpha_i u_i$.  Note that, if the components of $u_i$ are
$u_{i,1},...,u_{1,d}$, we can get $\hat{h}(x)$ by starting with
\[
T = \{ (u_{1,1},...,u_{1,d},h(u_1)),...,(u_{k,1},...,u_{k,d},h(u_k)) \}
\]
taking the convex combination of the members of $T$ with
mixing coefficients $\alpha_1,...,\alpha_k$, and then reading off
the last component.  Caratheodory's
theorem\footnote{Caratheodory's theorem states that if a point $x$ of $R^d$ lies in the convex hull
of a set $P$, then there is a subset $\hat{P}$ of $P$ consisting of $d+1$ or fewer points such that $x$
lies in the convex hull of $\hat{P}$.}
implies that we can get the same result using a mixture of
at most $d+1$ members of $T$.  In other words, we can assume without
loss of generality that $k = d+1$, so that
\begin{equation}
\label{e:dsum}
\hat{h}(x)=\max_{\sum_{i=1}^{d+1}{\alpha_i=1, \alpha_i\geq 0, x=\sum_{i=1}^{d+1} \alpha_i x_i}}{\sum_{i=1}^{d+1}\alpha_i h(x_i) }.
\end{equation}

Because of the case where $(\alpha_1,...,\alpha_{d+1})$ concentrates
all its weight on one component, we have $h(x) \leq \hat{h}(x)$.

We also claim that
\begin{equation}
\label{e:logbound}
h(x) \geq \hat{h}(x) - \beta \lceil \log_2 (d + 1) \rceil.
\end{equation}
We will prove this by induction on $\log_2 (d + 1)$, treating the
case in which $d+1$ is a power of $2$.  (By padding with zeroes if
necessary, we may assume without loss of generality that $d+1$ is a
power of $2$.)
The base case, in which $d = 1$, follows immediately from the definitions.
Let $k = d+1$.
Assume that
$x = a_1 x_1 + a_2 x_2 + ... + a_k x_k$,  $\sum_{i=1}^{k} a_i=1$, $a_i \geq 0$.
We can write this as:  $$x = (a_1+a_2)x_{1,2} + (a_3+a_4)x_{3,4} + ... (a_{k-1}+a_k)x_{k-1, k}$$
where $x_{i,i+1} = \frac{a_i}{a_i+a_{i+1}}x_i + \frac{a_{i+1}}{a_i+a_{i+1}}x_{i+1}$, for all $i$.
Now, by induction we have:
\begin{eqnarray*}
h(x) &\geq& -\beta \log(k/2) + (a_1+a_2)h(x_{1,2}) + \\
     & &  \hspace{0.2in} ... + (a_{k-1}+a_k)h(x_{k-1,k})\\
     &\geq& -\beta \log(k/2) \\
     & -& (a_1+a_2)\beta + a_1 h(x_1) + a_2 h(x_2) \\
                        &-& (a_3+a_4)\beta + a_3 h(x_3) + a_4 h(x_4) +  \\
                       && ... \\
                        &-& (a_{k-1}+a_k)\beta + a_{k-1} h(x_{k-1}) + a_k h(x_k) \\
      &=& -\beta \log(k) + a_1 h(x_1) + a_2 h(x_2) + a_3 h(x_3) + ... a_k h(x_k).
\end{eqnarray*}
The last inequality follows from the fact that $\sum_{i=1}^{n} a_i=1$.

So, we have proved (\ref{e:logbound}).
If we further normalize $e^{\hat{h}}$ to make it
 a density function, we obtain $\tilde{f}$ that is log-concave and satisfies
$f(x)/C \leq \tilde{f}(x) \leq Cf(x),$ where $C=e^{\beta \lceil \log_2 (d+1)\rceil}.$
This implies that for any $x$ we have $|f(x)-\tilde{f}(x)| \leq (C-1)\tilde{f}(x)$.

We now show that  the center of $\tilde{f}$ is close to the center of $f$.  We have:
\begin{eqnarray*}
\norm{\int {x(f(x) -\tilde{f}(x))dx}} &\leq& \int {\norm{x}|f(x) -\tilde{f}(x)|dx} \\ \leq (C-1) \int {\norm{x}\tilde{f}(x)dx}
&=& (C-1) \int_{r=0}^{\infty} {\P_{\tilde{f}}[\norm{X} \geq r] dr}.
\end{eqnarray*}
Using concentration properties of $\tilde{f}$ (in particular Lemma~\ref{isotropic-basic}) we get
\begin{eqnarray*}
\norm{\int {x(f(x) -\tilde{f}(x))dx}}
& \leq & (C-1) \int_{r=0}^{\infty}{e^{-\frac{r}{\sqrt{Cd}}+1}} dr \\
& = & e (C-1)\sqrt{C d},
\end{eqnarray*}
as desired.
\end{proof}

\begin{theorem}
\label{almost isotropic}
(i) Let $f:R^2 \rightarrow R$ be the density function of a log-concave distribution centered at
$z$  and with covariance matrix $A=\E_f[(X-z)(X-z)^T]$. Assume $f$ satisfies $\norm{z} \leq \xi$ and
$1/C \leq \int {(u \cdot x)^2 f(x) dx} \leq C$ for every unit vector $u$, for $C \geq 1$ constant close to $1$.
We have:
(a) Assume $(1/20+\xi)/ \sqrt{1/C-\xi^2} \leq 1/9$.
Then there exist an universal constant $c$ s.t.  we have $f(x) \geq c$, for all $x$ with $0 \leq ||x|| \leq 1/20$.
(b) Assume $C \leq 1+ 1/5$. There exist  universal constants $c_1$ and $c_2$ such that $f(x) \leq C_1 \exp(-C_2 || x ||)$ for all $x$.

(ii) Let $f:R\rightarrow R$ be the density function of a log-concave distribution centered at
$\xi$  with standard deviation $\sigma = \sqrt{\Var_f(X)}$.  Then $f(x) \leq 1/\sigma$ for all $x$.
 If furthermore $f$ satisfies  $1/C \leq
\E_f[X^2] \leq C$ for $C \geq 1$ and $\xi/ \sqrt{1/C-\xi^2} \leq 1/9$,  then we have $f(0) \geq c$ for some universal constant $c$.
\end{theorem}

\begin{proof}
(i)
Let $Y= A^{-1/2} (X-z)$. Then  $Y$ is a log-concave distribution in the isotropic position.
  Moreover, the density function of $g$ is given by $g(y)= \det(A^{1/2}) f(A^{1/2}y +z).$
Let $M= \E[X X^T]$.
We have $$A=\E[(X-z)(X-z)^T]=\E[X X^T]- z z^T= M- z z^T.$$ Also, the fact $1/C \leq \int {(u \cdot x)^2 f(x) dx} \leq C$ for every unit vector $u$ is equivalent to
$$1/C \leq u^T \E[X X^T] u \leq C$$ for every unit vector $u$. Using $v=(1,0)$, $v=(0,1)$, and $v=(1/\sqrt{2}, 1/\sqrt{2})$ we get that $M_{1,1} \in [1/C,C]$,
$M_{2,2} \in [1/C,C]$, and $M_{1,2}= M_{2,1} \in [1/C-C,C-1/C]$. We also have $\norm{z} \leq \xi$ and $\det(A^{1/2})=\sqrt{\det(A)}$. All these imply that
 $$ \sqrt{(1/C-\xi^2)^2 - (C-1/C)^2} \leq \det(A^{1/2}) \leq C.$$

 (a) For $x=A^{1/2} y +z$ we have $\norm{x-z}^2=(x-z)(x-z)^T=\norm{y}^2 v^T A v$, where $v=(1/\norm{y})y$ is a unit vector,  so
$\norm{y} \leq \norm{x-z}/\sqrt{1/C-\xi^2}$.
If $\norm{x} \leq 1/20$ we have $\norm{y} \leq 1/9$,
so by Lemma~\ref{isotropic-basic} we have $g(y) \geq c_1$, so $f(y) \geq c$, for some universal constants $c_1, c_2$, as desired.

(b) We have $f(x) = \frac{1}{\det{(A^{1/2})}} g(A^{-1/2}(x-z))$. By Lemma~\ref{isotropic-basic} (b) we have
\[
f(x) \leq \frac{1}{\det{A^{1/2}}} \exp{\left[-c \norm{A^{-1/2}(x-z)}\right]}.
\]
By the triangle inequality we further obtain:
$$f(x) \leq \frac{1}{\det{(A^{1/2)}}} \exp{\left[c \norm{A^{-1/2}z}\right]} \exp{\left[-c \norm{A^{-1/2}x}\right]}.$$

For $C \leq 1+ 1/5$, we can show that $\norm{A^{-1/2}x} \geq (1/\sqrt{2}) \norm{x}$. It is enough to show
$\norm{A^{-1/2}x}^2 \geq (1/2) \norm{x}^2$, or that $2 \norm{v}^2 \geq \norm{A^{1/2}v}$, where $v=A^{-1/2}x$ (so $x=A^{1/2}v$).
This is equivalent to $2 v^Tv \geq v^T A v$, which is true since the matrix $2 I - A$ is positive semi-definite.

(ii) Define $Y = (X-z)/\sigma$.  We have $\E[Y]=0$ and $\E[Y^2] = 1$.
The density $g$ of $Y$ is given by $g(y) = \sigma
f(\sigma y + z)$.
Now, since $g$ is isotropic and log-concave, we can apply Lemma~\ref{isotropic-basic}(e ) to $g$.
So $g(y) \leq 1$ for all $y$.  So, $\sigma
f(\sigma y + z) \leq 1$ for all $y$, which implies $f(x) \leq 1/\sigma$ for all $x$.
The second part follows as in Theorem~\ref{almost isotropic}.
\end{proof}

\subsection{More covariance matrices}

In this section, we extend Theorem~\ref{th:agg.margin} to the
case of arbitrary covariance matrices.

\begin{theorem}
\label{t:any.cov.active}
\label{T:any.cov.ACTIVE}
If all distributions in $\D$ are zero-mean and
log-concave in $R^d$, then arbitrary
$f \in \C$ be learned in polynomial time from
arbitrary distributions in $\cal D$
in the active learning model from
$O((d + \log(1/\delta) + \log \log (1/\epsilon))\log (1/\epsilon))$
labeled examples, and in the passive learning model from
$O\left(\frac{d + \log(1/\delta)}{\epsilon}\right)$ examples.
\end{theorem}

Our proof is through a series of lemma.  First, \citep{lv:2007} have shown
how to reduce to the nearly isotropic case.
\begin{lemma}[\citep{lv:2007}]
\label{l:lv.reduce}
For any constant $\kappa > 0$, there is a polynomial time
algorithm that, given polynomially many samples from
a log-concave distribution $D$, outputs an estimate $\Sigma$ of
the covariance matrix of $D$ such that, with probability
$1 - \delta$ the distribution
$D'$ obtained by sampling $x$ from $D$ and producing
$\Sigma^{-1/2} x$ has
$\frac{1}{1 + \kappa} \leq \E((u \cdot x)^2) \leq 1 + \kappa$
for all unit vectors $u$.
\end{lemma}

As a result of Lemma~\ref{l:lv.reduce}, we can assume without loss of
generality that the distribution $D$ satisfies
$\frac{1}{1 + \kappa} \leq \E((u \cdot x)^2) \leq 1 +
\kappa$ for an arbitrarily small constant $\kappa$.  By
Theorem~\ref{almost isotropic}, this implies that, without loss of
generality, there are constants $c_1,...,c_4$ such that, for
the density $f$ of any one or two-dimensional marginal $D'$ of $D$, we have
\begin{equation}
\label{e:dens.low}
f(x) \geq c_1 \mbox{ for all } x \mbox{ with } || x || \leq c_2,
\end{equation}
and for all $x$,
\begin{equation}
\label{e:dens.up}
f(x) \leq c_3 \exp(-c_4 || x ||).
\end{equation}

We will show that these imply that $\cal D$ is admissible.
\begin{lemma}
\label{l:angle.anycov}
(a) There exists $c$ such
that for any two unit vectors $u$ and $v$ in $\reals^d$ we have
$c \anglesep(v,u) \leq \distf_D(u,v).$

(b) For any $c_6 > 0$, there is a $c_7 > 0$ such that the
following holds.
Let $u$ and $v$ be two unit vectors in $R^d$, and assume that
$\theta(u,v) = \alpha < \pi/2$.  Then
\[
\P_{x \sim D} [ \mathrm{sign}(u \cdot x) \neq \mathrm{sign}(v \cdot x),
| v \cdot x| \geq c_7 \alpha]
 \leq c_6 \alpha.
\]
\end{lemma}

\begin{proof}
(a) Projecting $D$ onto a subspace can only reduce the norm of its mean,
and its variance in any direction.  Therefore, as in the proof of
Lemma~\ref{l:angle}, we may assume without loss of generality that $d = 2$.
Here, let us define
$A$ to be the region of disagreement between $u'$ and $v'$ intersected
with the ball $B_{c_2}$ of radius $c_2$ in $R^2$. Then we have
$\distf_{D_2}(u',v') \geq  \vol(A)  \inf_{x \in A} D_2(x)
 \geq   \vol(B_{c_2}) c_1 \anglesep(u,v).
$
(b) This proof basically amounts to observing that everything that was needed for
the proof of Theorem~\ref{lemma:vectors-sophist} is true for
$D$, because of (\ref{e:dens.low}) and (\ref{e:dens.up}).
\end{proof}
Armed with Lemma~\ref{l:angle.anycov},
to prove
Theorem~\ref{t:any.cov.active}, we can just apply Theorem~\ref{t:admissible}.

\section{Lower Bounds}
\label{a:lower}
The proof of our lower bounds (Theorem~\protect\ref{t:lower}) relies on a lower
bound on the packing numbers $M_{D}(\C,\epsilon)$. Recall that the
$\epsilon$-packing number, $M_{D}(\C,\epsilon)$, is the maximal
cardinality of an $\epsilon$-separated set with classifiers from $\C$,
where we say that $w_1,...,w_N$ are $\epsilon$-separated w.r.t $\D$ if
$\distf_D(w_i,w_j) >\epsilon$ for any $i \neq j$.
\begin{lemma}
\label{packing-logconcave}
There is a positive constant $c$ such that, for all $\beta < c$,
the following holds.
Assume that $D$ is $\beta$ log-concave in $R^d$,
and that its covariance matrix has full rank.
For all sufficiently small
$\epsilon$, $d \in N$, we have
$M_{D}(\C,\epsilon) \geq \frac{\sqrt{d}}{2}\left(\frac{c}{2 \epsilon}\right)^{d-1} -1.$
\end{lemma}
\begin{proof}
We first prove the lemma in the case that $D$ is isotropic.
The proof in this case follows the outline of a proof
for the special case of the uniform distribution in \citep{Lon95}.

Let $\Uball_d$ be the uniform distribution on the surface of the unit
ball in $\reals^d$.  By Theorem~\ref{l:angle.beta}, there exists $c$ such that
 for any two unit vectors $u$ and $v$ in $\reals^d$ we have $ c
\anglesep(v,u) \leq \distf_D(u,v).$
This implies that for a fixed $u$ the probability that a randomly
chosen $v$ has $\distf_D(u,v) \leq \epsilon$ is upper bounded by the volume
of those vectors in the interior of the unit ball whose angle is at
most $ \epsilon/c$ divided by the volume of the unit ball.  Using
known bounds on this ratio (see \citep{Lon95}) we have $\P_{v \in \Uball_d}[\distf_D(u,v)
  \leq \epsilon] \leq \frac{1}{\sqrt{d}} \left(\frac{2
  \epsilon}{c}\right)^{d-1}$, so $\P_{u,v \in \Uball_d}[\distf_D(u,v) \leq
  \epsilon] \leq \frac{1}{\sqrt{d}} \left(\frac{2
  \epsilon}{c}\right)^{d-1}$. That means that for a fixed $s$ if we
pick $s$ normal vectors at random from the unit ball, then  the
expected number of pairs of half-spaces that are $\epsilon$-close
according to $D$ is at most $ \frac{s^2}{\sqrt{d}} \left(\frac{2
  \epsilon}{c}\right)^{d-1}$. Removing one element of each pair from
$S$ yields a set of $ s- \frac{s^2}{\sqrt{d}} \left(\frac{2
  \epsilon}{c}\right)^{d-1}$ halfspaces that are
$\epsilon$-separated. Setting $ s=\frac{\sqrt{d}} {(2 \epsilon/
  c)^{d-1}}$, leads the desired result.

To handle the non-isotropic case, suppose that $\Sigma$ is
the covariance matrix of $D$, so that $\Sigma^{-1/2}$ is
the whitening transform.  Let $D'$ be the whitened version of
$D$, i.e.\ the distribution obtained by first choosing
$x$ from $D$, and then producing $\Sigma^{-1/2} x$.
We have
$\distf_D(v,w) = \distf_{D'}(v\Sigma^{1/2},w \Sigma^{1/2})$
(because $\sign(v \cdot x) \neq \sign(w \cdot x)$ iff
$\sign((v\Sigma^{1/2}) \cdot (\Sigma^{-1/2} x))
  \neq \sign((w\Sigma^{1/2}) \cdot (\Sigma^{-1/2} x))$).
So we can use an $\epsilon$-packing w.r.t.\ $D'$ to construct
an $\epsilon$-packing of the same size w.r.t.\ $D$.
\end{proof}

Now we are ready to prove Theorem~\ref{t:lower}.

\repeatclaim{Theorem~\ref{t:lower}}{\thmtlower}

\begin{proof}
First, let us consider passive PAC learning.
It is known \citep{Lon95} that, for any distribution $D$, the sample
complexity of passive PAC learning origin-centered linear
separators w.r.t. $D$ is at least
\[
\frac{d-1}{e} \left(\frac{M_D(\C,2\epsilon)}{4} \right)^{1/(d-1)}.
\]
Applying Lemma~\ref{packing-logconcave} gives an $\Omega(d/\epsilon)$
lower bound.
It is known \citep{Lon95} that, if for each $\epsilon$, there is a pair
of classifier $v,w$ such that $\distf_D(v,w) = \epsilon$, then
the sample complexity of PAC learning is
$\Omega((1/\epsilon)\log (1/\delta))$;  this requirement is
satisfied by $D$. 

Now let us consider the sample complexity of active learning.
As shown
in~\citep{Kulkarni:93}, in order to output a hypothesis of error at
most $\epsilon$ with probabality at least $1-\delta$, where $\delta
\leq 1/2$ and active learning algorithm that is allowed to make
arbitrary
yes-no queries must make $\Omega (\log M_{D}(\C,\epsilon))$ queries.
Using this together with Lemma~\ref{packing-logconcave} we get the desired result.
\end{proof}

\section{The inseparable case: Disagreement-based active learning}
\label{a:dis-coeff:beta}

\repeatclaim{Theorem~\ref{th:dis-coeff:beta}}{\thdisagreement}
\begin{proof}
Roughly, we will
show that almost all $x$ classified by a large enough margin by
$w^*$ are not in $\DIS(B(w^{*},r))$, because all hypotheses agree with
$w^*$ about how to classify such $x$, and therefore all pairs of hypotheses
agree with each other.
Consider $w$ such that $d(w, w^*) \leq r$; by Theorem~\ref{l:angle.beta} we
have $\theta(w, w^*) \leq c r $.
Define $C = e^{\beta \lceil \log_2 (d+1) \rceil}$ as in the proof
of Theorem~\ref{l:angle.beta}.
For any $x$ such that
$||x|| \leq \sqrt{d C} \log(1/r)$ we have
\begin{eqnarray*}
(w \cdot x - w^* \cdot x) & < & || w - w^* || \times || x || \\
                          & \leq & c r \sqrt{d C} \log(1/r).
\end{eqnarray*}
Thus, if $x$ also satisfies $|w^* \cdot x| \geq
c r \sqrt{d C} \log(1/r) $ we have $(w^* \cdot x) (w \cdot x) > 0$.
Since this is true for all $w$, any such $x$ is not in $\DIS(B(h,r))$.
By Theorem~\ref{l:angle.beta}
we have, for a constant $c_2$, that
\[
\P_{x \sim D} {( |w^* \cdot x| \leq c r \sqrt{C d} \log(1/r) )}
    \leq c_2  r \sqrt{C d}  \log(1/r).
\]
Moreover, by Theorem~\ref{l:angle.beta} we also have
\[
\P_{x \sim D} {[ ||x||  \geq  c r \sqrt{C d} \log(1/r)]} \leq r.
\]
These both imply
$\capacity_{w^*,D}(\epsilon)= O(C^{1/2} \sqrt{d} \log (1/\epsilon))$. 
\end{proof}

\section{Massart and Tsybakov noise}
\label{a:massart}
In this section we analyze label complexity for active learning under
the popular Massart and Tsybakov noise conditions, proving
Theorem~\ref{th:agg.margin:tsybakov}.

We consider a variant of the Tsybakov noise condition \citep{MT99}.
We assume that the classifier
$h$ that minimizes $\P_{(x,y) \sim D_{XY}}(h(x) \neq y)$
is a linear classifier, and that, for the weight vector $w^*$ of that
optimal classifier,
there exist known parameters $\alpha, \betats>0$ such that,
for all $w$, we have
 \begin{eqnarray}
 \betats (\distf_{D}(w,w^*))^{1/(1-\alpha)} \leq
\err(w)-\err(w^*).\label{eq:nonsep-tsybakov}
 \end{eqnarray}

By generalizing Theorem~\ref{lemma:vectors-sophist} so that it provides
a stronger bound for larger margins, and combining the result
with the other lemmas of this paper and techniques from \citep{BBZ07},
we get the following.

\repeatclaim{Theorem~\ref{th:agg.margin:tsybakov}}{\thmtsybakov}

Note that the case where $\alpha = 0$ is more general than the well-known
Massart noise condition \citep{MN06}.
In this case, for active learning,
Theorem~\ref{th:agg.margin:tsybakov} improves over the previously best
known results~\citep{steve-surrogate} 
by a  (disagreement coefficient) $\dis_{w^{*},D}(\epsilon)$ 
factor.
For passive learning, the
bound on the total number of examples needed improves by
$\log(\capacity_{w^{*},D}(\epsilon))$ factor the previously known best bound of~\citep{gine:06}.  It is consistent with recent lower
bounds of~\citep{RR11}  that include $\log(\capacity_{w^{*},D}(\epsilon))$ because those bounds
are for a worst-case domain distribution, subject to a constraint
on $\capacity_{w^{*},D}(\epsilon)$.

When $\alpha > 0$, the previously best result for active learning~\citep{steve-surrogate} is
\[
O((1/\epsilon)^{2 \alpha} \dis_{w^{*},D}(\epsilon) (d \log(\dis_{w^{*},D} (\epsilon)) + \log(1/\delta)).
\]
Combining this with our new bound on $\dis_{w^{*},D}(\epsilon)$
(Theorem~\ref{th:dis-coeff:beta}) we get
a bound of
\[
O((1/\epsilon)^{2 \alpha} d^{3/2}   \log(1/\epsilon) (\log(d) + \log \log (1/\epsilon))
  + \log(1/\delta))
\]
for log-concave distributions.
So our Theorem~\ref{th:agg.margin:tsybakov} saves roughly a
factor of $\sqrt{d}$, at the expense of
an extra  $\log (1/\epsilon)$ factor.

\smallskip
\noindent
We  note that the results in this section can also be extended to nearly log-concave distributions
by making use of our results in Section~\ref{se:near-log}.

\subsection{Proof of Theorem~\ref{th:agg.margin:tsybakov}}
We are now ready to discuss the proof of Theorem~\ref{th:agg.margin:tsybakov}.
As in \citep{BBZ07}, we will use a different algorithm in the
inseparable case (Algorithm~2).
\begin{algorithm}[tbh]
{\bf Input}:
a sampling oracle for $\distrib$, and a labeling oracle
a sequence of sample sizes $m_k>0$, $k\in \integers$;
a sequence of cut-off values $b_k >0$, $k\in \integers$
a sequence of hypothesis space radii $r_k >0$, $k\in \integers$;
a sequence of precision values $\epsilon_k >0$, $k\in \integers$

{\bf Output}: weight vector $\hat{w}_{\rounds}$.
\begin{itemize}
\setlength{\itemindent}{-4mm}
\item Pick random $\hat{w}_0$: $\|\hat{w}_0\|_2=1$.
\item  Draw $m_1$ examples from $\distrib_X$, label them and put into $W$.
\begin{itemize}
\setlength{\itemindent}{-6mm}
\item {\bf iterate} $k=1,\ldots, \rounds$
\begin{itemize}
\setlength{\itemindent}{-8mm}
\item find $\hat{w}_k \in B(\hat{w}_{k-1},r_k)$ ($\|\hat{w}_k\|_2=1$)
to approximately minimize training error:
$\sum_{(x,y) \in W} I (\hat{w}_k \cdot x y) \leq \min_{w \in B(\hat{w}_{k-1},r_k)} \sum_{(x,y) \in W} I(w \cdot x y) + m_k \epsilon_k$.
 \item clear the working set $W$
 \item
 until $m_{k+1}$ additional data points are labeled,
 draw sample $x$ from $\distrib_X$
 \begin{itemize}
\item if $|\hat{w}_{k}\cdot x| \geq b_k$, reject $x$
 \item otherwise, ask for label of $x$, and put into $W$
 \end{itemize}
\end{itemize}
{\bf end iterate}
\end{itemize}
\end{itemize}
 \label{fig:active-nonsep}
 \caption{Margin-based Active Learning
(non-separable case)}
\end{algorithm}

\subsubsection{Massart noise ($\alpha = 0$)}

We start by analyzing Algorithm~2 in the
case that $\alpha = 0$; the resulting assumption is more general than the
well-known Massart noise condition.

From the log-concavity assumption, the proof of
Theorem~\ref{th:agg.margin}, with slight modifications, proves that
there exists $c$ such that for all $w$ we have
\begin{equation}
\label{e:angle.up.massart}
c \betats \anglesep(w,w^*) \leq \err(w)-\err(w^*).
\end{equation}

We prove by induction on $k$ that after $k \leq s$ iterations, we
have $$\err(\hat{w}_k) -\err(w^*) \leq c \betats 2^{-k}$$ with
probability
$
1 - \frac{\delta}{2} \sum_{i < k} \frac{1}{(1 + s - i)^2}
$.
The case $k=1$ follows from
classic bounds~\citep{Vapnik:book98}.

Assume now the claim is true for $k-1$ ($k \geq
2$). Then at the $k$-th iteration, we can let $S_1=\{x:
|\hat{w}_{k-1}\cdot x| \leq b_{k-1}\}$ and $S_2=\{x:
|\hat{w}_{k-1}\cdot x| > b_{k-1}\}$. By induction hypothesis, we
know that with probability at least
$
1 - \frac{\delta}{2} \sum_{i < k - 1} \frac{1}{(1 + s - i)^2}
$
$\hat{w}_{k-1}$ has excess errors at most $c \betats 2^{-(k-1)}$,
implying, using (\ref{e:angle.up.massart}), that
$\anglesep(\hat{w}_{k-1},w^*) \leq 2^{-(k-1)}$.
By assumption, $\anglesep(\hat{w}_{k-1},\hat{w}_k) \leq 2^{-(k-1)}$.

From Theorem~\ref{lemma:vectors-sophist}, recalling that
$\betats$ is a constant, we have both:
\begin{eqnarray*}
&& \Pr((\hat{w}_{k-1} \cdot x) (\hat{w} \cdot x) < 0, x \in S_2)
  \leq c \betats 2^{-k}/4  \\
&& \Pr((\hat{w}_{k-1} \cdot x) (w^* \cdot x) < 0, x \in S_2)
  \leq c \betats 2^{-k}/4.
\end{eqnarray*}

Taking the sum, we obtain:
\begin{equation}
\Pr((\hat{w} \cdot x) (w^* \cdot x) < 0, x \in S_2)
  \leq c \betats 2^{-k}/2.
\end{equation}
Therefore:
\begin{eqnarray*}
\err(\hat{w}_k)-\err(w^*) &\leq& (\err(\hat{w}_k|S_1)-\err(w^*|S_1))
\Pr(S_1) \\ &+& \Pr((\hat{w} \cdot x) (w^* \cdot x) < 0, x \in S_2) \\
&\leq & (\err(\hat{w}_k|S_1)-\err(w^*|S_1)) c_3 b_{k-1} \\ &+& c \betats 2^{-k}/2.
\end{eqnarray*}
By standard Vapnik-Chervonenkis bounds, we can choose $C$ s.t.\ with $m_{k}$
samples, we obtain $$\err(\hat{w}_k|S_1)-\err(w^*|S_1)  \leq  c \betats 2^{-k}/(c_3 b_{k-1})$$ with probability $1-(\delta/2) /(1 + s - i)^2$.
Therefore $\err(\hat{w}_k) -\err(w^*) \leq c  \betats 2^{-k}$ with
probability $
1 - \frac{\delta}{2} \sum_{i < k} \frac{1}{(1 + s - i)^2}
$, as desired.

The bound on the total number of examples, labeled and unlabeled,
follows the same line of argument as Theorem~\ref{t:passive}, except
with the constants of this analysis.

\subsubsection{Tsybakov noise ($\alpha > 0$)}
We now treat the more general Tsybakov noise.

For this analysis, we need a generalization of
Theorem~\ref{lemma:vectors-sophist} that provides a stronger
bound on the probably of large-margin errors, using a stronger
assumption on the margin.

\begin{theorem}
\label{lemma:vectors-sophist.bbig}
There is a positive constant $c$ such that the following holds.
Let $u$ and $v$ be two unit vectors in $R^d$, and assume that
$\theta(u,v) = \eta < \pi/2$.  Assume that $D$
is isotropic log-concave in $R^d$.  Then, for any
$b \geq c \eta$, we have
\begin{equation}
\label{e:largemargin.bbig}
\P_{x \sim D} [ \mathrm{sign}(u \cdot x) \neq \mathrm{sign}(v \cdot x)
\mbox{ and }
| v \cdot x| \geq b]
 \leq C_5  \eta \exp(- C_6 b/\eta),
\end{equation}
for absolute constants $C_5$ and $C_6$.
\end{theorem}
\begin{proof}
Arguing as in the proof of Lemma~\ref{l:angle}, we may assume without loss of
generality that $d = 2$.

Next, we claim that each member $x$ of $E$ has
$|| x || \geq b/\eta$.  Assume without loss of generality
that $v \cdot x$ is positive.  (The other case is symmetric.)  Then
$u \cdot x < 0$, so the angle of $x$ with $u$ is
obtuse, i.e. $\theta(x,u) \geq \pi/2$.
Since $\theta(u,v) = \eta$, this implies that
\begin{equation}
\label{e:theta.big.bbig}
\theta(x,v) \geq \pi/2 - \eta.
\end{equation}
But $x \cdot v \geq b$, and $v$ is unit length, so
$|| x || \cos \theta(x,v) \geq b$, which, using (\ref{e:theta.big.bbig}),
implies
$
|| x || \cos (\pi/2 - \eta) \geq b,
$
which, since
$\cos (\pi/2 - \eta) \leq \eta$ for all $\eta \in [0,\pi/2]$, in turn
implies
$
|| x || \geq b/\eta.
$
This implies that, if $B(r)$ is a ball of radius $r$ in $\R^2$, that
\begin{equation}
\label{e:shells.bbig}
\P[ E ] = \sum_{i=1}^{\infty}
                \P[ E \cap (B((i+1) (b/\eta)) - B(i (b/\eta))) ].
\end{equation}
Let us bound one of the terms in RHS.  Choose $i \geq 1$.

Let $f(x_1,x_2)$ be the density of $D$.  We have
\begin{eqnarray*}
&& \P[ E \cap (B((i+1) (b/\eta)) - B(i (b/\eta))) ] \\
&&  =
  \int_{(x_1,x_2) \in B((i+1) (b/\eta)) - B(i (b/\eta))} 1_E(x_1,x_2) f(x_1,x_2) \; dx_1 dx_2.
\end{eqnarray*}

Let $R_i= B((i+1) (b/\eta)) - B(i (b/\eta)$.
Applying the density upper bound from Lemma~\ref{isotropic-basic}
with $d=2$, there are constants $C_1$ and $C_2$ such that
\begin{eqnarray*}
&& \P[ E \cap (B((i+1) (b/\eta)) - B(i (b/\eta))) ] \\
&& \leq
  \int_{(x_1,x_2) \in R_i}
     1_E(x_1,x_2) C_1 \exp(- (b/\eta) C_2 i) dx_1 dx_2 \\
&& =
 C_1 \exp(- (b/\eta) C_2 i) \cdot \\
&&     \int_{(x_1,x_2) \in R_i} 1_E(x_1,x_2) \; dx_1 dx_2.
\end{eqnarray*}
If we include $B(i (b/\eta))$ in the integral again, we get
\begin{eqnarray*}
&& \P[ E \cap (B((i+1) (b/\eta)) - B(i (b/\eta))) ] \\
&& \leq
 C_1 \exp(- (b/\eta) C_2 i)  \int_{(x_1,x_2) \in B((i+1) (b/\eta))} 1_E(x_1,x_2) \; dx_1 dx_2.
\end{eqnarray*}
Now, we exploit the fact that the integral above is a rescaling of
a probability with respect to the uniform distribution.
Let $C_3$ be the volume of the unit ball in $\R^2$.  Then, we have
\begin{eqnarray*}
&& \P[ E \cap (B((i+1) (b/\eta)) - B(i (b/\eta))) ] \\
&& \leq
 C_1 \exp(- (b/\eta) C_2 i)  C_3 (i+1)^2 (b/\eta)^2 \eta/\pi \\
&& = C_4 (b/\eta)^2 \eta (i+1)^2 \exp(-(b/\eta) C_2 i ),
\end{eqnarray*}
for $C_4 = C_1 C_3 /\pi$.  Returning to (\ref{e:shells.bbig}), we
get
\begin{eqnarray*}
\P[ E ] & = & \sum_{i=1}^{\infty}  C_4 (b/\eta)^2 \eta (i+1)^2 \exp(-(b/\eta) C_2 i)  \\
& = & C_4  (b/\eta)^2 \eta \sum_{i=1}^{\infty} (i+1)^2 \exp(-(b/\eta) C_2 i) \\
& = & C_4  (b/\eta)^2 \times \frac{4 e^{2 (b/\eta) C_2} - 3 e^{(b/\eta) C_2} + 1}{\left(e^{(b/\eta) C_2} - 1 \right)^3} \times \eta.
\end{eqnarray*}
Now, if $b/\eta > 4/C_2$, we have
\begin{eqnarray*}
\P[ E ] & \leq &
   C_4  (b/\eta)^2 \times \frac{5 e^{2 (b/\eta) C_2}}{\left(e^{(b/\eta) C_2}/2\right)^3} \times \eta \\
   & \leq &
   C_5  \eta \times (b/\eta)^2 \exp(-(b/\eta) C_2) \mbox{(where
                                  $C_5 = 40 C_4$)} \\
   & = &
   C_5  \eta \times \exp(-(b/\eta) C_2 + 2 \ln (b/\eta)) \\
   & \leq &
   C_5  \eta \times \exp(-(b/\eta) C_2/2),
\end{eqnarray*}
completing the proof.
\end{proof}

Now we are ready to prove Theorem~\ref{th:agg.margin:tsybakov} in
the case that $\alpha > 0$.

Under the noise condition~\ref{eq:nonsep-tsybakov} and from the log-concavity assumption, we obtain that there exists $c$ such that for all $w$ we have:
\[
\betats c^{1/(1-\alpha)} \anglesep(w,w^*)^{1/(1-\alpha)} \leq \err(w)-\err(w^*).
\]
Let us denote by $\tc=\betats c^{1/(1-\alpha)}$. For all $w$, we have:
\begin{equation}
\label{e:angle.up.tsybakov}
\tc \anglesep(w,w^*)^{1/(1-\alpha)} \leq \err(w)-\err(w^*).
\end{equation}

We prove by induction on $k$ that after $k \leq s$ iterations,
we have
\[
\err(\hat{w}_k) -\err(w^*) \leq \tc 2^{-k}
\]
with probability
$
1 - \frac{\delta}{2} \sum_{i < k} \frac{1}{(1 + s - i)^2}.
$
The case $k=1$ follows from classic bounds.

Assume now the claim is true for $k-1$ ($k \geq
2$). Then at the $k$-th iteration, we can let $S_1=\{x:
|\hat{w}_{k-1}\cdot x| \leq b_{k-1}\}$ and $S_2=\{x:
|\hat{w}_{k-1}\cdot x| > b_{k-1}\}$. By the induction hypothesis, we
know that with probability at least $1- \delta \sum_{i < k - 1} \frac{1}{(1 + s - i)^2}$,
$\hat{w}_{k-1}$ has excess errors at most $\tc 2^{-(k-1)(1-\alpha)}$,
implying
\[
\anglesep(\hat{w}_{k-1},w^*) \leq 2^{-(k-1)(1-\alpha)}.
\]
By assumption, $\anglesep(\hat{w}_{k-1},\hat{w}_k) \leq 2^{-(k-1)(1-\alpha)}$.

Applying Theorem~\ref{lemma:vectors-sophist.bbig}, we have both:
\begin{eqnarray*}
&& \Pr((\hat{w}_{k-1} \cdot x) (\hat{w} \cdot x) < 0, x \in S_2)
  \leq \tc 2^{-k}/4  \\
&& \Pr((\hat{w}_{k-1} \cdot x) (w^* \cdot x) < 0, x \in S_2)
  \leq \tc  2^{-k}/4
\end{eqnarray*}
Taking the sum, we obtain:
\begin{equation}
\Pr((\hat{w} \cdot x) (w^* \cdot x) < 0, x \in S_2)
  \leq \tc  2^{-k}/2.
\end{equation}
Therefore:
\begin{eqnarray*}
\err(\hat{w}_k)-\err(w^*) &\leq& (\err(\hat{w}_k|S_1)-\err(w^*|S_1))
\Pr(S_1) \\ && + \Pr((\hat{w} \cdot x) (w^* \cdot x) < 0, x \in S_2) \\
&\leq & (\err(\hat{w}_k|S_1)-\err(w^*|S_1)) b_k
   \\ && + \tc 2^{-k}/2.
\end{eqnarray*}

By standard bounds, we can choose $C_1$, $C_2$ and
$C_3$ s.t.\ with $m_{k}$
samples, we obtain
$\err(\hat{w}_k|S_1)-\err(w^*|S_1)
   \leq \epsilon_k \leq \frac{\tc 2^{-k}}{2 b_k}$
with probability $1- (\delta/2) /(1 + s - i)^2$.
Therefore $\err(\hat{w}_k) -\err(w^*) \leq \tc 2^{-k}$ with
probability $1 - \frac{\delta}{2} \sum_{i < k} \frac{1}{(1 + s - i)^2}$,
as desired, completing the proof of Theorem~\ref{th:agg.margin:tsybakov}.

\end{document}